\documentclass[final,5p,times,twocolumn]{elsarticle}

\usepackage{amssymb}
\usepackage{amsthm}
\usepackage{amsmath}

\usepackage{lineno}

\usepackage{array}
\usepackage{booktabs}
\usepackage{graphicx}
\usepackage{subfigure}
\usepackage{multirow}
\usepackage{verbatim}

\newcommand{\PreserveBackslash}[1]{\let\temp=\\#1\let\\=\temp}
\newcolumntype{C}[1]{>{\PreserveBackslash\centering}p{#1}}
\newcolumntype{R}[1]{>{\PreserveBackslash\raggedleft}p{#1}}
\newcolumntype{L}[1]{>{\PreserveBackslash\raggedright}p{#1}}

\journal{Pattern Recognition}

\begin{document}

\begin{frontmatter}



\title{Revisiting Competitive Coding Approach for Palmprint Recognition: \\ A Linear Discriminant Analysis Perspective}


\author{Lingfei Song}
\author{Hua Huang\corref{*}}
\cortext[*]{Corresponding author. E-mail addresses: lingfei@bit.edu.cn (L. Song), huahuang@bit.edu.cn (H. Huang).}

\address{School of Computer Science and Technology, Beijing Institute of Technology, Beijing, 100081, P.R. China.}

\begin{abstract}
	Competitive Coding approach (CompCode) is one of the most promising methods for palmprint recognition. Due to its high performance and simple formulation, it has been continuously studied for many years. However, although numerous variations of CompCode have been proposed, a detailed analysis of the method is still absent. 
	In this paper, we provide a detailed analysis of CompCode from the perspective of linear discriminant analysis (LDA) at the first time. A non-trivial sufficient condition under which the CompCode is optimal in the sense of Fisher's criterion is presented. 
	Based on our analysis, we examined the statistics of palmprints and conclude that CompCode deviates from the optimal condition. To mitigate the deviation, we propose a new method called ``Class-Specific CompCode" that improves CompCode by excluding non-palm-line areas from matching. A non-linear mapping on the competitive code is also applied in this method to further enhance accuracy. 
	Experiments on two public databases demonstrate the effectiveness of the proposed method.
\end{abstract}


\end{frontmatter}

\graphicspath{{figures/}}

%---------------------------------------------------------------------------------------------------
\section{Introduction} \label{sec:introduction}
Palmprint contains many stable and discriminative features, which make it suitable for personal identification \cite{jain2004introduction}. According to the scale, palmprint features can be categorized into ridge based ones and winkle and principle line based ones. Ridge-based palmprint features, such as singular points and minutiae points, however, can only be obtained from high-resolution palmprint images (${>}$400dpi), whereas wrinkle and principle line based palmprint features, such as the orientation of wrinkles and principle lines, can be obtained from low-resolution images (${<}$100dpi) \cite{zhang2003online}. 
Early works on palmprint recognition usually concentrate on high-resolution palmprint images due to its important role in law enforcement \cite{ wu1997pyramid, shu1998automated, zhang1999two, shu2001automatic, duta2002matching, you2002hierarchical,  han2003personal}. However, for civil and commercial applications, low-resolution palmprint images are more suitable, because they are easy to capture using CCD cameras. 

The first seminal method that works on low-resolution palmprint images is PalmCode \cite{zhang2003online}, which is inspired by IrisCode for iris recognition \cite{daugman1993high, daugman2003random, daugman2007new, daugman2009iris}. PalmCode employs a single Gabor filter to extract the local phase information of palmprints and then uses the Hamming distance to measure the similarity between two PalmCodes at the matching stage. The success of PalmCode encourages a huge amount of researches in coding-based palmprint recognition methods \cite{ adams2004competitive, jia2008palmprint, kong2009a, guo2009palmprint, zuo2010multiscale, fei2016double-orientation, fei2016half,  xu2018discriminative}. Among them, the most popular and effective one is Competitive Coding (CompCode) approach \cite{adams2004competitive}. 
Instead of the local phase information, CompCode extracts the orientation of palm lines as palmprint features. Since the palmprint is full of lines and wrinkles with rich distinctive orientation information, CompCode stays as the most successful palmprint recognition method in many years \cite{guo2009palmprint, fei2016double-orientation, xu2018discriminative}. Figure~\ref{fig:flowchart} displays the working flow of CompCode. For two palmprints to compare, CompCode first extracts the orientation features of them, then, at the matching stage, calculates the angular distance between these two palmprints. The decision whether these two palmprints belong to the same person will be made according to the matching distance.

CompCode has attracted a lot of research attention over the past decade, and numerous variations of CompCode have been proposed. The most representative works include robust line orientation code method (RLOC) \cite{jia2008palmprint}, binary orientation co-occurrence vector method (BOCV) \cite{guo2009palmprint}, sparse multi-scale competitive code approach (SMCC) \cite{zuo2010multiscale}, double orientation code method (DOC) \cite{fei2016double-orientation}, half orientation code method (HOC) \cite{fei2016half}, discriminative and robust competitive code approach (DRCC) \cite{xu2018discriminative}, and so on. These methods improve CompCode in a way of more robustly extracting the orientation features of palmprints. 

However, despite these numerous variations of CompCode, a detailed analysis of the method is still absent, which makes it unclear why the CompCode is so effective and how to further improve it. Specifically, CompCode makes a decision simply by calculating the angular distance between two palmprints, the question is whether the decision method is optimal under a certain criterion. If so, which the criterion is, if not, how to improve it? Due to CompCode's foundation role in palmprint recognition, clarifying these two questions will definitely benefit the design of new coding-based methods. 

\begin{figure}[t]
	\centering
	\includegraphics[width=1.0\linewidth]{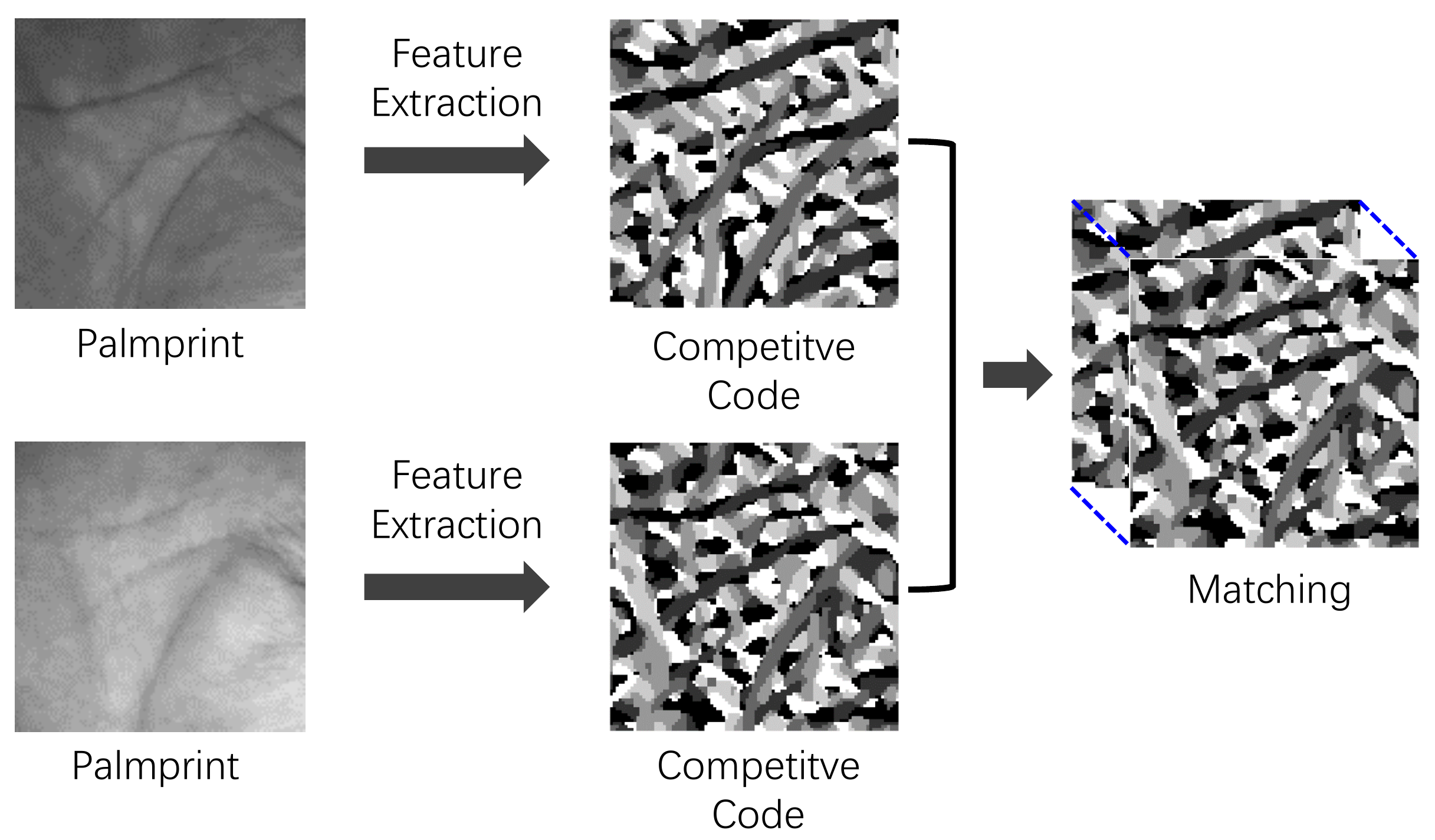}
	
	\caption{The working flow of CompCode. For two palmprints to compare, the orientation features are first extracted, then the distance between these two palmprints is calculated at the matching stage.}
	
	\label{fig:flowchart}
\end{figure}

In this paper, we provide a detailed analysis of CompCode from the perspective of linear discriminant analysis (LDA) at the first time. Our focus is on the matching stage with the assumption that the features of palmprints have been robustly extracted. This approach makes sense because feature extraction has been steadily improved in prior works whereas the matching almost remains the same in all the coding-based methods. We show that the matching in CompCode is actually equal to a linear classifier that classifies the palmprints into two classes by a hyperplane. The normal direction of the hyperplane is, by accident, the diagonal of a high dimensional space,  since the matching treats each pixel (or dimension) in an indiscriminative way. Note that such a classifier is not learned from training data but is empirically determined by authors (of CompCode). 

To test the optimality of the classifier, we first give a non-trivial sufficient condition under which the above classifier is optimal in the sense of Fisher's criterion. Our result is that if the input vector (feature) has a constant mean with respect to indexes and a stationary variance, then the linear classifier with a weight vector of the diagonal of the input vector space is optimal in the sense of Fisher's criterion. Then, by examining the statistics of palmprints, we find that the classifier used in CompCode, though nearly approach, deviates from the optimal condition. This may account for the effectiveness of CompCode and gives room to improvement.

In order to mitigate the deviation, we propose a new coding-based method called ``Class-Specific CompCode" (CSCC), which improves CompCode by excluding non-palm-line areas from the CompCode matching. In addition, we also studied the optimal coding scheme in CompCode, which shows that a non-linear mapping on the competitive code can further enhance the accuracy. Experiments on public databases demonstrate the effectiveness of the proposed method. Besides, our experiments also shows that the proposed two improvement strategies can benefit other coding-based methods.

The remainder of this paper is organized as follows. Section~\ref{sec:related} reviews Competitive Coding (CompCode) approach and other coding-based palmprint recognition methods. Section~\ref{sec:lda} gives an analysis of CompCode from the view of LDA. We propose a new coding-based palmprint method in Section~\ref{sec:improvement}. Section~\ref{sec:experiments} demonstrates the effectiveness of the proposed method on public databases. Finally, Section~\ref{sec:conclusion} concludes this paper.

%-----------------------------------------------------------------------------------------------------
\section{Related Work} \label{sec:related}
In this section, we will give a detailed review on CompCode. Other representative coding-based palmprint recognition methods will be briefly reviewed, since their computational architectures are the same as that of CompCode. 

\subsection{Competitive Coding (CompCode) approach}
CompCode uses six Gabor filters to capture the orientation features of palm lines. The captured orientation features are encoded into six integers, i.e., ${\{0,1,2,3,4,5\}}$, called Competitive Code. At the matching stage, the angular distance between two Competitive Codes are calculated. The decision whether these two palmprints belong to the same person is made according to the matching distance. 

Let's first look at the feature extraction. The general form of the Gabor filter is
\begin{equation}\label{eq:1}
	\begin{aligned}
		G(x,y,\theta,u,\sigma) = &\frac{1}{2\pi\sigma^2} \exp\{-\frac{x^2+y^2}{2\sigma^2}\} \cdot \\
								   &\exp\{2\pi i u(x \cos\theta + y \sin\theta)\},
	\end{aligned}
\end{equation}
where ${x}$, ${y}$ are the coordinates of 2-D space, ${\theta}$ controls the orientation of the function, ${u}$ is the frequency of the sinusoidal wave, and ${\sigma}$ is the standard deviation of the Gaussian envelope. To make it more robust against brightness, a discrete Gabor filter, ${G[i,j,\theta,u,\sigma]}$, is turned to zero DC (direct current) with the application of the following formula
\begin{equation}\label{eq:2}
	\bar{G}[i,j,\theta,u,\sigma] = G[i,j,\theta,u,\sigma] - 
									 \frac{\sum_{i=-n}^{n}{\sum_{j=-n}^{n}{G[i,j,\theta, u,\sigma]}}}{(2n+1)^2},
\end{equation}
where ${(2n+1)^2}$ is the size of the filter. The adjusted Gabor filter is used to filter the pre-processed palmprint images.

It should be noted that Gabor filter's ability to capture the orientation information of palm lines is based on the assumption that palm lines have an upside-down Gaussian shape, which is given by
\begin{equation}\label{eq:3}
	L(x,y)=A \cdot [1 - \exp(-\frac{x^2}{2\sigma_{L}^2})] + B,
\end{equation}
where ${A}$ controls the magnitude of the line, which depends on the contrast of palmprint images, ${\sigma_{L}}$ is the standard deviation of the profile, which can be considered as the width of palm lines, and ${B}$ is the brightness of the line, which relies on the brightness of palmprint images. 

The filter response on the middle of the line, i.e., ${x=0}$, is
\begin{equation}\label{eq:4}
R(\delta\theta)=\bar{G}\otimes L \propto 
-A \exp \{-\pi^2 u^2\frac{2\sigma^4}{\sigma^2+\sigma_{L}^{2}}\sin^2\delta\theta\},
\end{equation}
where ${\otimes}$ represents two dimensional convolution, ${\delta \theta}$ is the difference between the orientation of palm line and Gabor filter. According to Eq.~\eqref{eq:4}, the filter response value ${R(\delta \theta)}$ reaches minimum when filter has the same orientation as palm line, that is, ${\delta\theta=\pi/2}$. Six discrete orientations, e.g., ${k\pi/6, (k=0,1,\dots,5)}$, are used in CompCode. The Competitive Code ${C}$ of a palmprint image is defined as 
\begin{equation}\label{eq:5}
	C[i,j] = \arg \min_{k}\{I[i,j] \otimes \bar{G}[i,j,\theta_{k},u,\sigma]\},
\end{equation}
where ${I}$ is a pre-processed palmprint image, and ${\otimes}$ represents 2-D convolution. 

At the matching stage, the distance between two palmprints is calculated. Let ${C_1}$, ${C_2}$ be Competitive Codes of two palmprints ${PLM_1}$, ${PLM_2}$, respectively. The matching difference ${\Delta}$ between these two Competitive Codes is defined by pixel-to-pixel angular distance, that is,
\begin{equation} \label{eq:6}
	\Delta[i,j] = \min\{|C_{1}[i,j]-C_{2}[i,j]|, 6-|C_{1}[i,j] - C_{2}[i,j]|\}.
\end{equation}
Finally, the distance between two palmprints is calculated by averaging among all the points in ${\Delta}$, that is, 
\begin{equation}\label{eq:7}
	d(PLM_{1},PLM_{2})=\frac{\sum_{i=0}^{N-1}{\sum_{j=0}^{N-1}{\Delta[i,j]}}}{N^2},
\end{equation}
where ${N^2}$ is the size of palmprint image.

The matching is counted as \emph{genuine matching} if both samples are from the same person, otherwise, it is viewed as \emph{impostor matching}. Figure~\ref{fig:compcode_dis} displays genuine and impostor matching distance of CompCode.  Generally, the genuine matching distance is much smaller than impostor matching distance. Thus, we can simply set a threshold to determine whether two palmprints belongs to the same person. 

\subsection{Other Coding Based Methods}
Based on CompCode, many similar coding-based methods have been proposed \cite{jia2008palmprint, guo2009palmprint, zuo2010multiscale, fei2016double-orientation, fei2016half, xu2018discriminative}. These methods share the same computational architecture as CompCode but improve the CompCode in a way of more robustly extracting the orientation features. We will give a brief summary of these methods in the following.

To robustly extract orientation features, robust line orientation code method (RLOC) \cite{jia2008palmprint} adopts modified finite Randon transform, instead of Gabor filters, in feature extraction. Binary orientation co-occurrence vector method (BOCV) \cite{guo2009palmprint} states that using only one dominant orientation to represent a local region may lose some valuable information and proposes a novel feature extraction algorithm to represent multiple orientations for a local region. To go further, sparse multi-scale competitive code approach (SMCC) \cite{zuo2010multiscale} proposes a compact representation of multiscale palm line orientation features. The SMCC method first defines a filter bank of second derivatives of Gaussians with different orientations and scales and then uses the ${l_{1}}$-norm sparse coding to obtain a robust estimation of the multiscale orientation field. Double orientation code method (DOC) \cite{fei2016double-orientation} proposes a novel double-orientation code scheme to represent the orientation features using two dominant orientations. Half orientation code method (HOC) \cite{fei2016half} improves DOC by a bank of ``half-Gabor" filters. Recently, a discriminative and robust competitive code approach (DRCC) \cite{xu2018discriminative} combines the orientation features with a side code to accurately and robustly represent the palmprints.

Although CompCode has been improved to some extent by these methods, a detailed analysis of the method is still absent, which makes it unclear why the CompCode is so effective and how to further improve it. 
In the next section, we will try to explain CompCode from the view of linear discriminant analysis (LDA). Based on our analysis, an improved version of CompCode will be presented in Section~\ref{sec:improvement}.

\begin{figure}[t]
	\centering
	\includegraphics[width=0.96\linewidth]{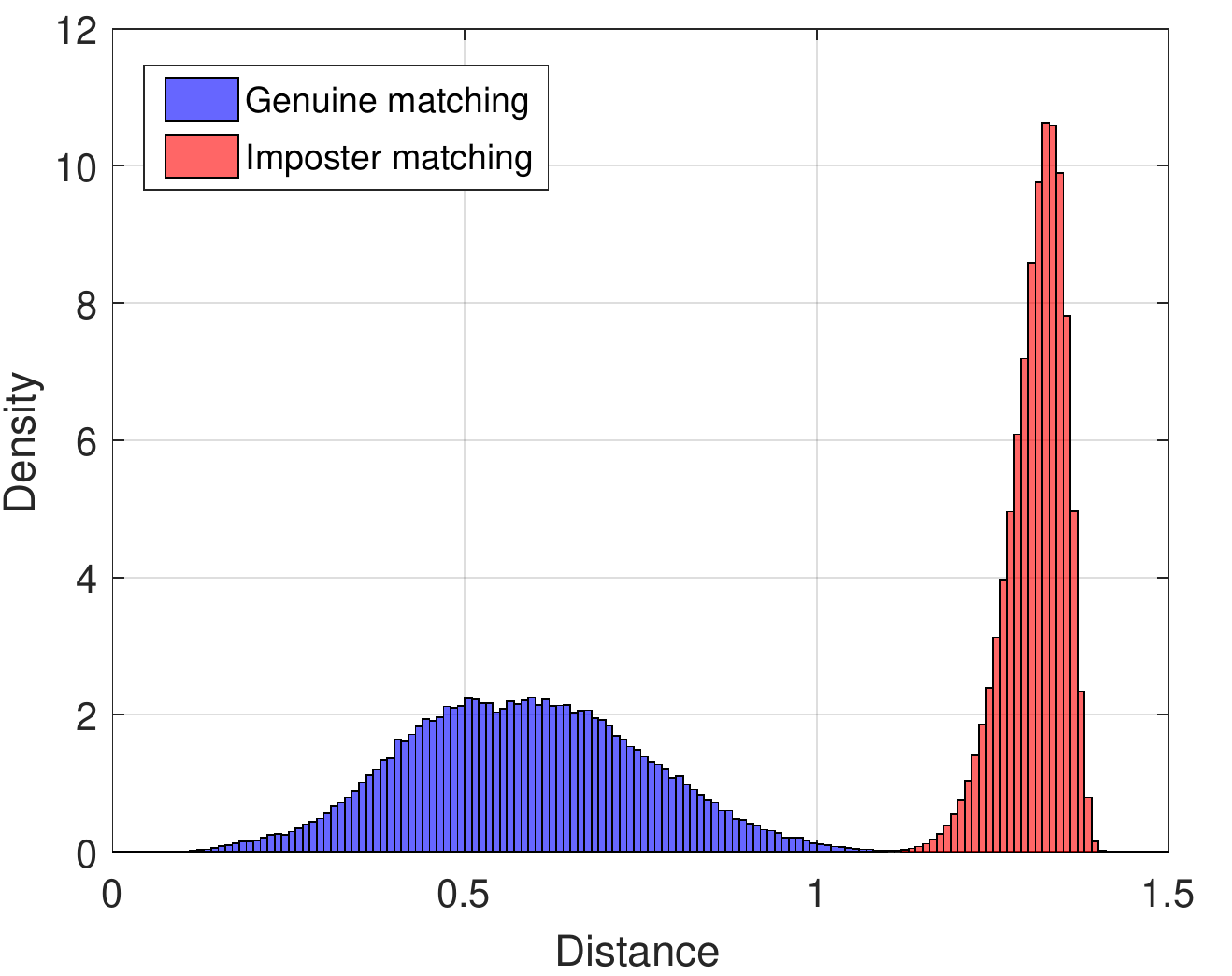}
	
	\caption{The genuine and impostor matching distance of CompCode.}
	\label{fig:compcode_dis}
\end{figure}

%-------------------------------------------------------------------------------------------------
\section{Understanding CompCode from the View of LDA} \label{sec:lda}
In this section, we will analyze CompCode from the perspective of Linear Discriminant Analysis (LDA). Our focus is on the matching stage of CompCode assuming that the features have been robustly extracted. We find that the matching in CompCode is actually a linear classifier with a weight vector ${[1,1,...,1]^{\intercal}_{\tiny N\times N}}$. To test the optimality of this classifier, a non-trivial sufficient condition under which the classifier is optimal in the sense of Fisher's criterion is presented, and the statistics of palmprints is examined.

\subsection{Interpretation as a Linear Classifier}
Our analysis starts from Eq.~\eqref{eq:6}, which defines the matching difference ${\Delta}$ between two palmprints. For convenience, we call this matching difference as genuine matching difference if the matching is genuine, otherwise, the matching difference is termed as impostor matching difference.

Figure~\ref{fig:features} shows an example of the genuine matching difference and impostor matching difference. It can be seen that the genuine matching difference and impostor matching difference are very unlike. This is the foundation for palmprint recognition. Generally speaking, the genuine matching difference has lower pixel values than the impostor matching difference. Thus, it is tempting to categorize a matching as genuine or impostor according to the average pixel value of their matching difference. This is exactly what CompCode does. 

As previously mentioned, CompCode calculates the distance between two palmprints by averaging all the pixels in their matching difference (Eq.~\eqref{eq:7}).
Then, by simply setting a threshold, the matching is counted as genuine or impostor. This averaging method is very straightforward with an advantage of easy-to-understand. However, it involves no statistics of palmprints at all, which makes further analysis very difficult.

\begin{figure}[t]
	\centering
	\subfigure[Genine matching]{
		\includegraphics[width=0.415\linewidth]{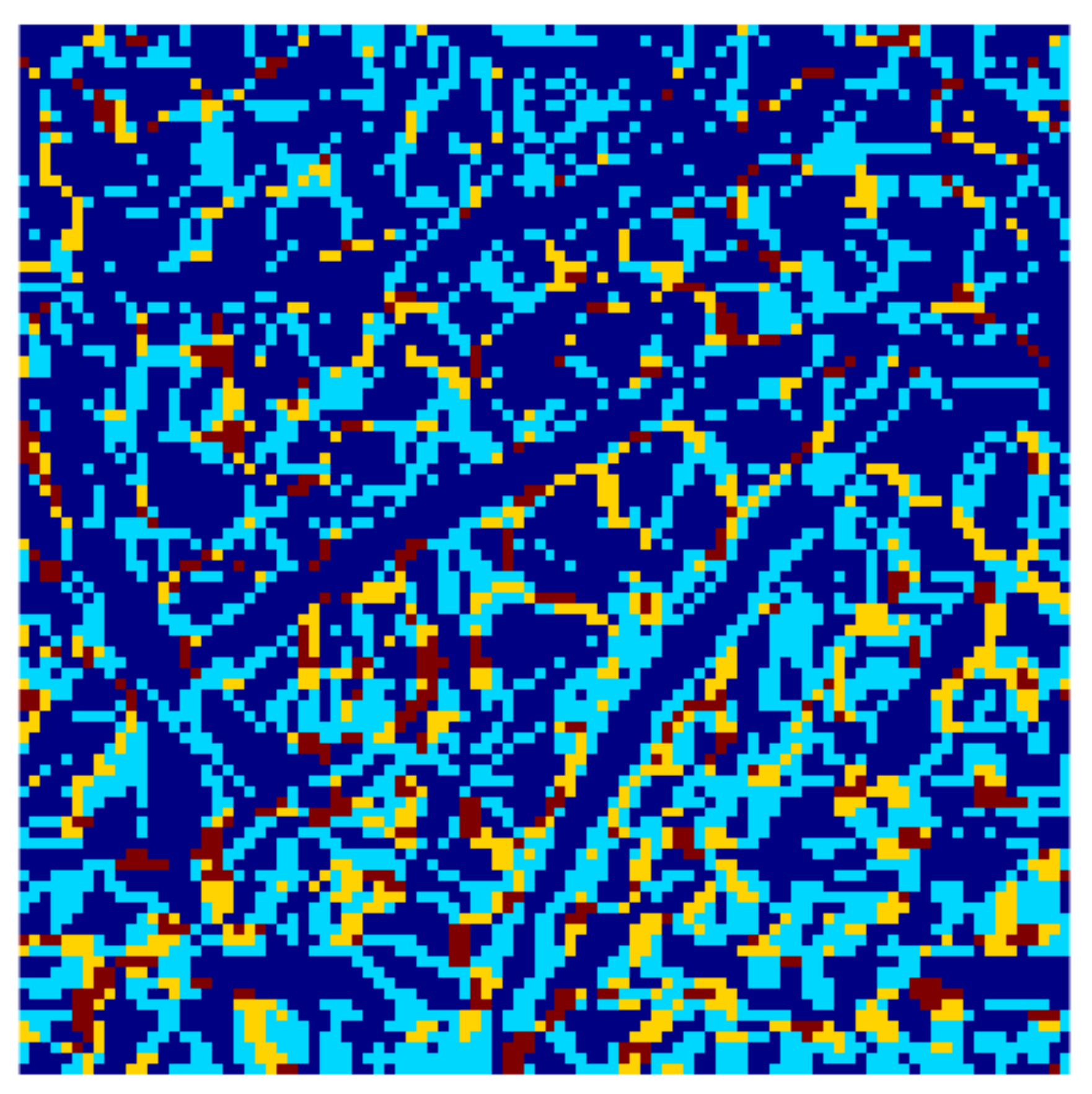}
	}
	\subfigure[Impostor matching ~~~~~~~]{
		\includegraphics[width=0.466\linewidth]{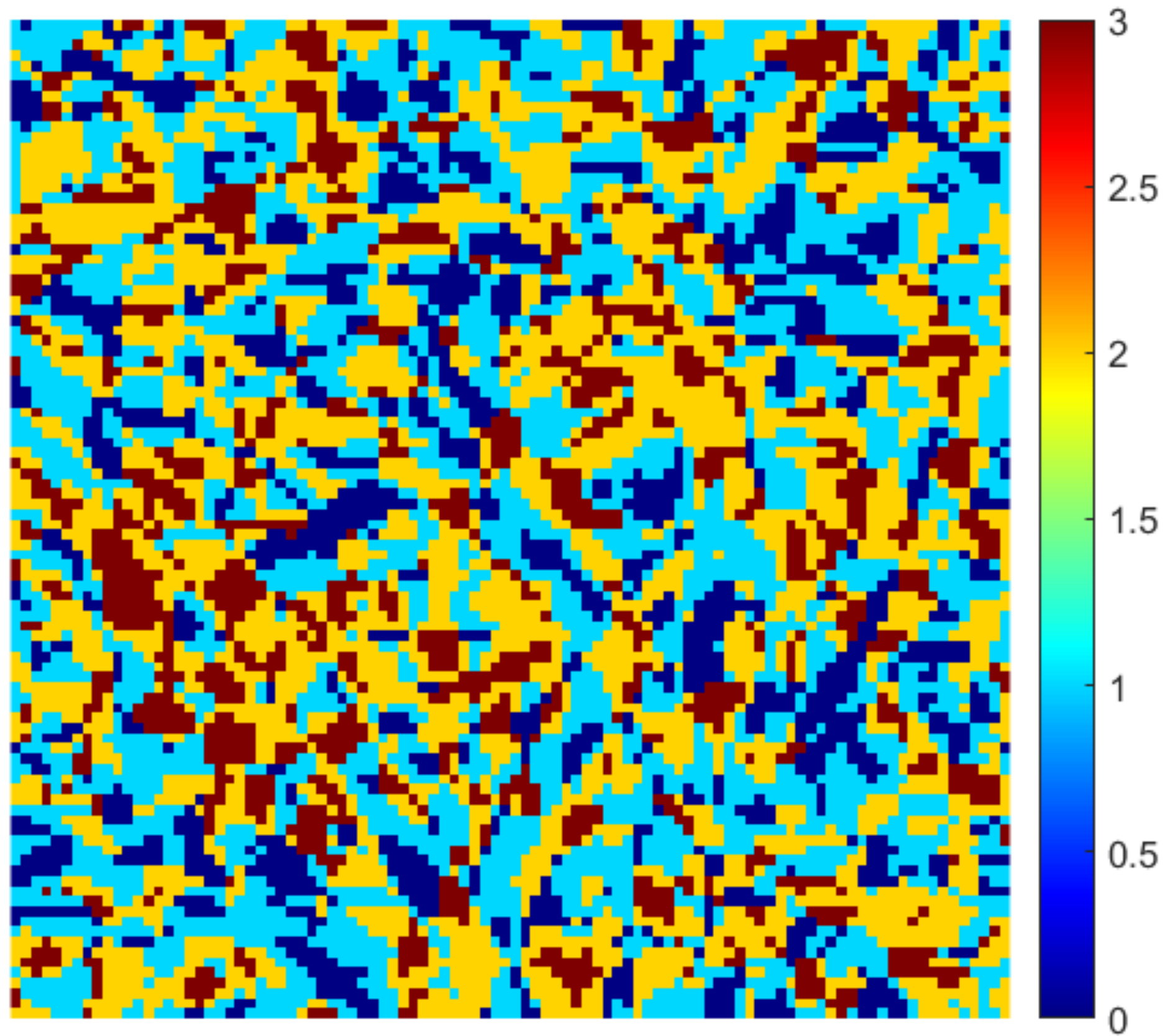}
	}
	\caption{An example of the genuine and impostor matching difference.}
	\label{fig:features}
\end{figure}

By simple manipulations of symbols in Eq.~\eqref{eq:7}, it will be clear that the matching in CompCode is actually equal to a linear classifier with a weight vector ${[1,1,...,1]^{\intercal}_{N \times N}}$, i.e., 
\begin{equation} \label{eq:9}
	g(\Delta) = \frac{1}{N^2} \cdot [1,1,...,1]_{N \times N} \cdot \text{Vec}(\Delta),
\end{equation}
where ${\text{Vec}(\cdot)}$ vectorizes the two-dimensional matching difference ${\Delta}$. Eq.~\eqref{eq:9} associates the above averaging method, which is used in CompCode, with a linear classifier. It can be seen that classification by a hyperplane is equivalent to projecting the features along the normal direction of the hyperplane. 
Figure~\ref{fig:lda} shows a low dimensional example. On the one hand, we can use a hyperplane to categorize the genuine matching and the impostor matching. On the other hand, we can first project the features along the normal direction of the hyperplane and then simply use a threshold to make a decision. These two are equivalent.

It should be noted that the weight vector ${[1,1,...,1]^{\intercal}_{N \times N}}$ used in CompCode is not learned from training data but is empirically determined by the author. This raises a question of whether this linear classifier is optimal under a certain criterion. Answering this question needs to first choose a criterion. The way to qualify the weight vector in a linear classifier is the well known Fisher's criterion \cite{friedman2001elements}, which states that along the direction of the weight vector the ratio of the inter-class variance to the intra-class variance should be maximized. 

Denoting the genuine matching difference as ${\Delta_{\text{genu}}}$, the impostor matching difference as ${\Delta_{\text{impo}}}$, a weight vector ${\mathbf{w}}$ that satisfies Fisher's criterion obeys the following equation \cite{friedman2001elements}
\begin{equation} \label{eq:10}
	\mathbf{w} \propto (\Sigma_{\text{genu}} + \Sigma_{\text{impo}})^{-1}(\mathbf{\mu}_{\text{genu}} - \mathbf{\mu}_{\text{impo}}),
\end{equation}
where 
\begin{equation}
	\begin{aligned}
		\Sigma_{\text{genu}}~ =& ~~\text{Var}[\text{Vec}(\Delta_{\text{genu}})] \\
		\Sigma_{\text{impo}}~ =& ~~\text{Var}[\text{Vec}(\Delta_{\text{impo}})] \\
		\mathbf{\mu}_{\text{genu}}~ =& ~~~~~\text{E}[\text{Vec}(\Delta_{\text{genu}})] \\
		\mathbf{\mu}_{\text{impo}}~ =& ~~~~~\text{E}[\text{Vec}(\Delta_{\text{impo}})],
	\end{aligned}
\end{equation}
in which ${\text{Var}[\cdot]}$ represents co-variance and ${\text{E}[\cdot]}$ represents expectation. 

Eq.~\eqref{eq:10} shows that whether ${[1,1,...,1]^{\intercal}_{N \times N}}$ is optimal in the sense of Fisher's criterion depends on the statistical properties of ${\Delta_{\text{genu}}}$ and ${\Delta_{\text{impo}}}$. In the next sub-section, we will give a non-trivial sufficient condition under which the weight vector ${[1,1,...,1]^{\intercal}_{N \times N}}$ is optimal. 

\begin{figure}[t]
	\centering
	\includegraphics[width=0.96\linewidth]{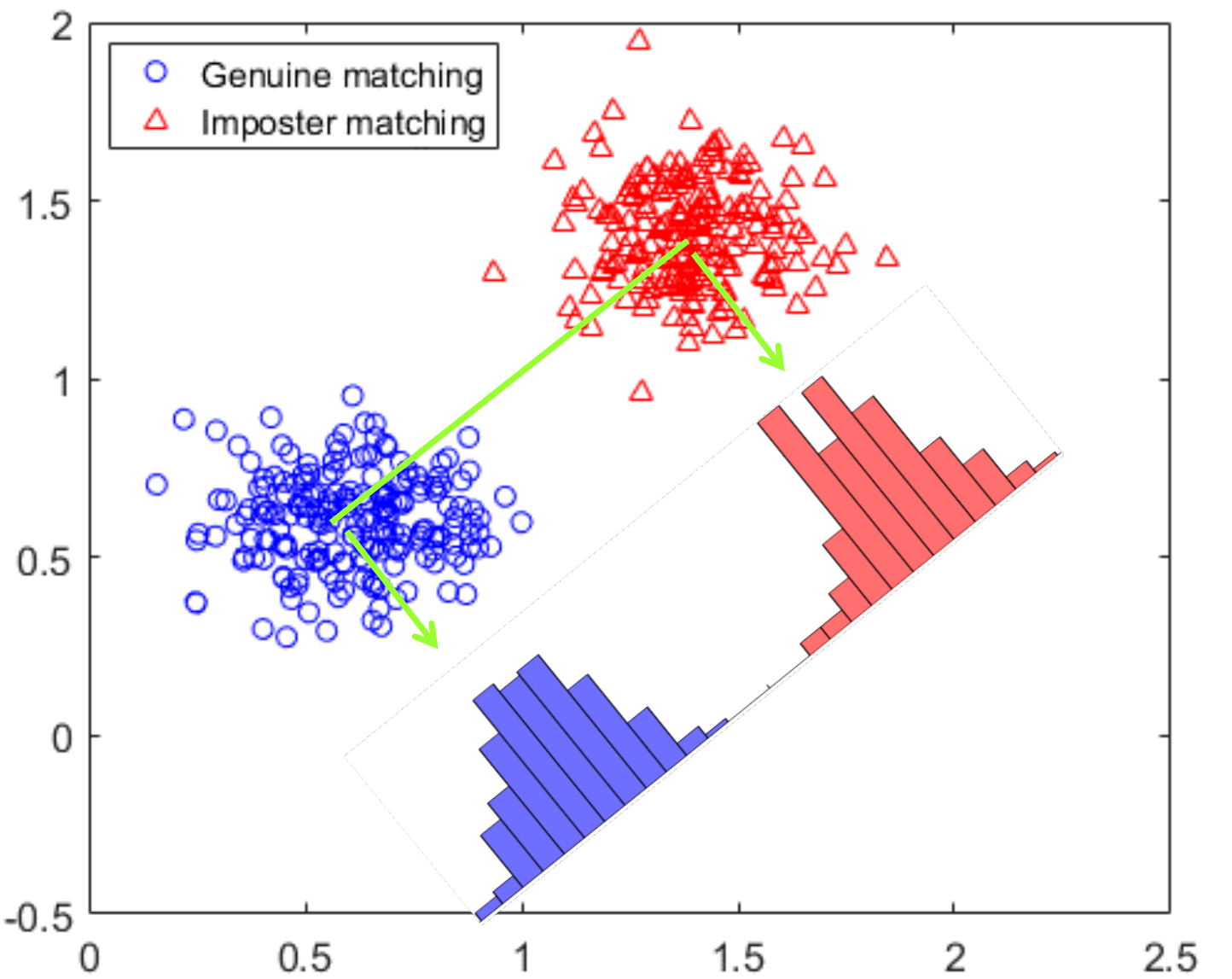}
	\caption{A low dimensional example of linear classification. Classification by a hyperplane is equivalent to projecting the features along the normal direction of the hyperplane and then setting a threshold.}
	\label{fig:lda}
\end{figure}

\subsection{A Non-trivial Sufficient Condition}
Before giving a non-trivial sufficient condition, it may be convenient to first consider a trivial condition to intuitively figure out how the statistical properties of ${\Delta_{\text{genu}}}$ and ${\Delta_{\text{impo}}}$ decides the optimal weight vector ${\mathbf{w}}$.

Suppose that the elements in ${\Delta_{\text{genu}}}$ are independent and identically distributed, the elements in ${\Delta_{\text{impo}}}$ are also independent and identically distributed. These two distributions are different. Then, the optimal weight vector is obviously ${\mathbf{w}=[1,1,...,1]^{\intercal}_{N \times N}}$. The proof of this conclusion is very straightforward. First, since the elements are identically distributed, ${\mathbf{\mu}_{\text{genu}} - \mathbf{\mu}_{\text{impo}}}$ is obviously proportional to ${[1,1,...,1]^{\intercal}_{N \times N}}$. Second, because the elements are independent and identically distributed, the co-variance matrices (${\Sigma_{\text{genu}}}$ and ${\Sigma_{\text{impo}}}$) are proportional to the unit matrix, and their inverse matrices are also proportional to the unit matrix. Thus, the optimal weight vector ${\mathbf{w}}$ is proportional to the multiply of a unit matrix and a vector of ${[1,1,...,1]^{\intercal}_{N \times N}}$, which is ${[1,1,...,1]^{\intercal}_{N \times N}}$. Actually, even if the above proof is missing, we can also get the right answer according to our intuition because there is no reason to treat the elements discriminatively since they are independent and identically distributed. 

The independent and identically distributed condition is so strict that is actually no use in practice. A quick glance at Figure~\ref{fig:features} will suggest that the elements in ${\Delta_{\text{genu}}}$ and ${\Delta_{\text{impo}}}$ are not independent and identically distributed. We need to relax the condition to make it approach the real distribution. 

We noticed that Fisher's criterion only involves the first-order (expectation) and the second-order (variance) properties of the distribution, thus, it may be possible to construct a non-trivial sufficient condition that only constrains the expectation and co-variance of ${\Delta_{\text{genu}}}$ and ${\Delta_{\text{impo}}}$. Indeed, we find that if the expectation of ${\Delta_{\text{genu}}}$ (and ${\Delta_{\text{impo}}}$) is constant with respect to indexes, and the variance of ${\Delta_{\text{genu}}}$ (and ${\Delta_{\text{impo}}}$) is stationary, then the optimal weight vector is ${\mathbf{w}=[1,1,...,1]^{\intercal}_{N \times N}}$. This non-trivial sufficient condition can be more precisely expressed as the following proposition.

\newtheorem{theorem}{\bf Proposition}
\begin{theorem} \label{Theorem:1}
	\rm{If ${\Delta_{\text{genu}}}$ and ${\Delta_{\text{impo}}}$ satisfy the following conditions:}
	\begin{enumerate}
		\item[${\cdot}$] ${\text{E}(\Delta_{\text{genu}}[i,j]) = m_{\text{genu}}}$,
		\item[${\cdot}$] ${\text{E}(\Delta_{\text{impo}}[i,j]) = m_{\text{impo}}}$,
		\item[${\cdot}$] ${\text{E}((\Delta_{\text{genu}}[i,j]-m_{\text{genu}}) (\Delta_{\text{genu}}[i^{'},j^{'}] - m_{\text{genu}})) = \sigma_{\text{genu}}(\delta i, \delta j)}$, 
		\item[${\cdot}$] ${\text{E}((\Delta_{\text{impo}}[i,j]-m_{\text{impo}}) (\Delta_{\text{impo}}[i^{'},j^{'}] - m_{\text{impo}})) = \sigma_{\text{impo}}(\delta i, \delta j)}$,
	\end{enumerate}
	%where ${\delta i = |i^{'}-i|}$ and ${\delta j = |j^{}|}$
	then ${(\Sigma_{\text{genu}} + \Sigma_{\text{impo}})^{-1}(\mathbf{\mu}_{\text{genu}} - \mathbf{\mu}_{\text{impo}}) \propto [1,1,...,1]^{\intercal}_{N \times N}}$.
	%in which ${\Sigma_{\text{genu}} = \text{Var}(\text{Vec}(\Delta_{\text{genu}}))}$, and  ${\Sigma_{\text{impo}} = \text{Var}(\text{Vec}(\Delta_{\text{impo}}))}$
\end{theorem}

\begin{proof}
	First, since the expectation is constant, it it obvious that 
	\begin{equation}
		\begin{aligned}
			\mathbf{\mu}_{\text{genu}}-\mathbf{\mu}_{\text{impo}} &= \text{E}[\text{Vec}(\Delta_{\text{genu}})] - \text{E}[\text{Vec}(\Delta_{\text{impo}})] \\
			& \propto [1,1,...,1]^{\intercal}_{N \times N}.
		\end{aligned}
	\end{equation}
	
	Second, according to the definition of the co-variance matrix ${\Sigma_{\text{impo}}}$, we will show that ${\Sigma_{\text{impo}}}$ is approximately a block-circulant matrix as follows. 
	\begin{equation} \label{eq:11}
		\begin{aligned}
			\Sigma_{\text{impo}} &= \text{E}[\text{Vec}(\Delta_{\text{impo}} - m_{\text{impo}}) \cdot \text{Vec}(\Delta_{\text{impo}} - m_{\text{impo}})^{\intercal}]		\\
				   =& \left [ 
				   			\begin{matrix}
				   			\small
				   			%B_{0}  & B_{1}  & \cdots & B_{M}  & 0      & 0      & 0     & 0 & \cdots & 0 \\
				   			%B_{-1} & B_{0}  & B_{1}  & \cdots & B_{M}  & 0      & 0     & 0 & \cdots & 0 \\ 
				   			%\cdots & B_{-1} & B_{0}  & B_{1}  & \cdots & B_{M}  & 0     & 0 & \cdots & 0 \\ 
				   			%B_{-M} & \cdots & B_{-1} & B_{0}  & B_{1}  & \cdots & B_{M} & 0 & \cdots & 0
				   			
				   			B_{0,0}  & B_{0,1}  & \cdots  &         &         &         &        &              \\
				   			B_{1,0}  & B_{1,1}  & B_{1,2} & \cdots  &         &         &        &              \\ 
				   			\cdots   & B_{2,1}  & B_{2,2} & B_{2,3} &  \cdots &         &        &              \\ 
				   			         &          & \ddots  & \ddots  & \ddots  &         &        &              \\
				   			         &          &         & \ddots  & \ddots  & \ddots  &        &              \\
				   			         &          &         &         &         & \cdots  & B_{N-1,N-2} & B_{N-1,N-1}
				   			\end{matrix}
				   	  \right] ,
		\end{aligned}
	\end{equation}
	where ${B_{i,j} = \text{E}\{(\Delta_{\text{impo}}[:,i] - m_{\text{impo}}) \cdot (\Delta_{\text{impo}}[:,j] - m_{\text{impo}})^{\intercal}\} }$ is the co-variance of the ${i}$th column and ${j}$th column of ${\Delta_{\text{impo}}}$. Since the variance of ${\Delta_{\text{impo}}}$ is stationary, ${B_{i,j}}$ only depends on the distance between ${i}$ and ${j}$, i.e., ${B_{i,j} = B_{|i-j|}}$. Therefore, Eq.~\eqref{eq:11} can be reformed as 
	\begin{equation} \label{eq:12}
			\Sigma_{\text{impo}} = \left [ 
						\begin{matrix}
						B_{0}      & B_{1}   & \cdots  &         &         &         &        &              \\
						B_{1}      & B_{0}   & B_{1}   & \cdots  &         &         &        &              \\ 
						\cdots     & B_{1}   & B_{0}   & B_{1}   & \cdots  &         &        &              \\ 
						           &         & \ddots  & \ddots  & \ddots  &         &        &              \\
						           &         &         & \ddots  & \ddots  & \ddots  &        &              \\
						           &         &         &         &         & \cdots  & B_{1}  & B_{0}
						\end{matrix}
					\right] .
	\end{equation}
	By little modification of Eq.~\eqref{eq:12}, it shows that ${\Sigma_{\text{impo}}}$ can be approximated by a block-circulant matrix
	\begin{equation} \label{eq:15}
		\Sigma_{\text{impo}} \approx \left [ 
					\begin{matrix}
					B_{0}      & B_{1}   & \cdots  &         &         &         & \cdots & B_{1}        \\
					B_{1}      & B_{0}   & B_{1}   & \cdots  &         &         &        & \cdots       \\ 
					\cdots     & B_{1}   & B_{0}   & B_{1}   & \cdots  &         &        &              \\ 
							   &         &         & \ddots  & \ddots  & \ddots  &        &              \\
							   &         &         &         & \ddots  & \ddots  & \ddots &              \\
					B_{1}      & \cdots  &         &         &         & \cdots  & B_{1}  & B_{0}
					\end{matrix}
				\right] .
	\end{equation}
	The same method can be used to derive each block ${B_{\delta}}$ (${\delta = 0,1,\cdots }$), which shows that ${B_{\delta}}$ can be approximated by a circulant matrix, i.e., 
	\begin{equation} \label{eq:16}
		B_{\delta} \approx \left [ 
						\begin{matrix}
							b^{\delta}_{0} & b^{\delta}_{1} & \cdots         &                &        &        & \cdots         & b^{\delta}_{1} ~\\
							b^{\delta}_{1} & b^{\delta}_{0} & b^{\delta}_{1} & \cdots         &        &        &                & \cdots \\ 
							\cdots         & b^{\delta}_{1} & b^{\delta}_{0} & b^{\delta}_{1} & \cdots &        &                & \\ 
								           &                &                & \ddots         & \ddots & \ddots &                & \\
									       &                &                &                & \ddots & \ddots & \ddots         & \\
							b^{\delta}_{1} & \cdots         &                &                &        & \cdots & b^{\delta}_{1} & b^{\delta}_{0}
						\end{matrix}
					\right] .
	\end{equation}
	Thus, ${\Sigma_{\text{impo}}}$ is a block-circulant matrix. Similarly, it can be proved that ${\Sigma_{\text{genu}}}$ is also a block-circulant matrix. 
    So far, we know that both ${\Delta_{\text{genu}}}$ and ${\Delta_{\text{impo}}}$ are block-circulant matrices. Obviously, their summation is also block-circulant. 
    
    Without loss of generality, let ${\Sigma = \Sigma_{\text{genu}} + \Sigma_{\text{impo}}}$. Then we will prove that the inverse of ${\Sigma}$ is a block-circulant matrix.
    Applying eigenvalue decomposition on ${\Sigma}$, we get 
    \begin{equation}
    	\Sigma = V \Lambda V^{H},
    \end{equation}
    in which ${\Lambda = \text{diag}(\lambda_1, \lambda_2, \cdots, \lambda_{N\times N})}$ is a diagonal matrix, ${V}$ is an orthogonal matrix whose ${(\alpha N + \beta)}$th column defined by stacking each row of the following matrix 
    \begin{equation}
    	v[k,l] = \exp[2\pi i (\frac{\alpha k + \beta l}{N})], \quad k, l = 0, 1, \cdots, N-1,
    \end{equation}
    with ${\alpha = 0,1,\cdots, N-1}$, ${\beta = 0,1,\cdots, N-1}$, and ${V^H}$ is the conjugate transpose of matrix ${V}$. According to the properties of eigenvalue decomposition, the inverse of ${\Sigma}$ is 
    \begin{equation} \label{eq:13}
    	\Sigma^{-1} = V \Lambda^{-1} V^H,
    \end{equation}
    where ${\Lambda^{-1}}$ is a diagonal matrix with its diagonal elements being the inverse of the diagonal elements in ${\Lambda}$. Eq.~\eqref{eq:13} can be further written as 
    \begin{equation}
		\Sigma^{-1} = \sum_{j = 0}^{N \times N - 1} \lambda_j^{-1} \cdot V[:,j] \cdot V[:,j]^{H}.
    \end{equation}
    It is clear that each item ${V[:,j] \cdot V[:,j]^{H}}$ is a block-circulant matrix due to the periodicity of sinusoidal function. Thus, ${\Sigma^{-1}}$ is also a block-circulant matrix. 
    
    Finally, the multiply of a block-circulant matrix ${\Sigma^{-1}}$ and a vector of ${[1,1,...,1]^{\intercal}_{N \times N}}$ is proportional to ${[1,1,...,1]^{\intercal}_{N \times N}}$ because the elements in each row of ${\Sigma^{-1}}$ are the same. 
\end{proof}

Proposition~\ref{Theorem:1} gives a sufficient condition under which the weight vector ${\mathbf{w} = [1,1,...,1]^{\intercal}_{N \times N}}$ is optimal in the sense of Fisher's criterion. Although this condition is not necessary, we find it very useful in testing the optimality of CompCode and designing new coding-based palmprint recognition methods. In the next subsection, we will test the optimality of CompCode by investigating the statistics of palmprints.

\subsection{Statistics of palmprints}
Our investigations on the statistics of palmprints show that the impostor matching difference ${\Delta_{\text{impo}}}$ satisfies the constant-expectation and stationary-variance condition, whereas the genuine matching difference ${\Delta_{\text{genu}}}$ does not.

\textbf{Constant expectation of ${\Delta_{\text{impo}}}$}. 
We first look at the expectation of the impostor matching difference ${\Delta_{\text{impo}}}$. The expectation of ${\Delta_{\text{impo}}}$ is estimated by averaging the results of ${10^4}$ independent experiments. Figure~\ref{fig:impo_mean} shows the averaging result of ${\Delta_{\text{impo}}}$. It can be seen that the expectation of ${\Delta_{\text{genu}}}$ is almost constant. Figure~\ref{fig:histogram_impo_mean} displays the histogram of the average ${\Delta_{\text{impo}}}$. It is clear that the average of ${\Delta_{\text{impo}}[i,j]}$ concentrates around ${1.5}$. 

We wish to model ${\Delta_{\text{impo}}[i,j]}$ by the following distribution. According to Eq.~\eqref{eq:6}, ${\Delta_{\text{impo}}[i,j] \in \{0,1,2,3\}}$ is a function defined on ${\Omega \times \Omega} = \{(\omega_{1},\omega_{2}) \colon \omega_{1}, \omega_{2} \in \Omega \}$, where ${\Omega = \{0, 1, 2, 3, 4, 5\}}$ contains each possible orientation of palm lines. Since the palmprints of difference person varies a lot, it may be safe to say that each orientation appears with equi-probability. Based on this equi-probability assumption, we can immediately give an probability distribution function of ${\Delta_{\text{impo}}[i,j]}$, which is provided in Table~\ref{table:1}. The empirical distribution out of ${10^4}$ experiments is that ${P_{e}(0) = 0.1607}$, ${P_{e}(1) = 0.3394}$, ${P_{e}(2) = 0.3462}$, and ${P_{e}(3) = 0.1537}$, which are very close to the theoretical results given in Table~\ref{table:1}. According to this model, the expectation of ${\Delta_{\text{impo}}[i,j]}$ is 1.5, consisting with the result in Figure~\ref{fig:histogram_impo_mean}.

\begin{table}[h]
	\centering
	\caption{Distribution function of the discrete random variable ${\Delta_{\text{impo}}[i,j]}$.}
	
	\begin{tabular}{p{1.8cm}<\centering p{1.1cm}<\centering p{1.1cm}<\centering p{1.1cm}<\centering p{1.1cm}<\centering p{1.1cm}<\centering}
		\toprule
		${\Delta_{\text{impo}}[i,j]}$ &  0   &  1  &  2  &  3  \\ 
		\midrule
		${P}$                         & 1/6  & 1/3 & 1/3 & 1/6 \\	
		\bottomrule
	\end{tabular}
	
	\label{table:1}
\end{table}

\textbf{Stationary variance of $\Delta_{\text{impo}}$}. We now investigate the variance of the impostor matching difference ${\Delta_{\text{impo}}}$. Our aim is to show that the variance is stationary. One way to test the stationary of the variance is observing the co-variance of ${\text{Vec}(\Delta_{\text{impo}})}$, where ${\text{Vec}(\cdot)}$ vectorizes the 2-D ${\Delta_{\text{impo}}}$. According to our analysis in Proposition~\ref{Theorem:1}, if the variance of ${\Delta_{\text{impo}}}$ is stationary, the co-variance matrix ${\text{Var}(\text{Vec}(\Delta_{\text{impo}}))}$ is a block circulant matrix except the boundary areas. The empirical co-variance matrix out of ${10^4}$ independent experiments is shown in Figure~\ref{fig:impo_var}. It can be seen that the co-variance matrix can be approximated by a block-circulant matrix, which mean that the variance of ${\Delta_{\text{impo}}}$ is stationary. 

Another way to test the stationary of ${\Delta_{\text{impo}}}$ is by investigating the distribution of the impostor matching distance defined by
\begin{equation}
d_{\text{impo}} = \frac{\sum_{i=0}^{N-1}\sum_{j=0}^{N-1} \Delta_{\text{impo}}[i,j]}{N^2}.	
\end{equation}
Making use of the central limit theorem, if each ${\Delta_{\text{impo}}[i,j]}$ is identically distributed and  ${\Delta_{\text{impo}}}$ is stationary, the average of the variables in ${\Delta_{\text{impo}}}$ could be approximated by a normal distribution \cite{feller2008introduction}. Figure~\ref{fig:distribution_impo} displays the empirical distribution of the impostor matching distance out of ${4.9 \times 10^5}$ independent experiments. The black curve is the fitted Gaussian curve. It can be seen that the distribution can be fitted by a Gaussian curve very well. The mean of the empirical distribution is 1.4914, closely approaches the theoretical value given by ${\text{E}(\Delta_{\text{impo}}[i,j]) = 1.5}$. 

\begin{figure}[t]
	\centering
	\includegraphics[width=0.92\linewidth]{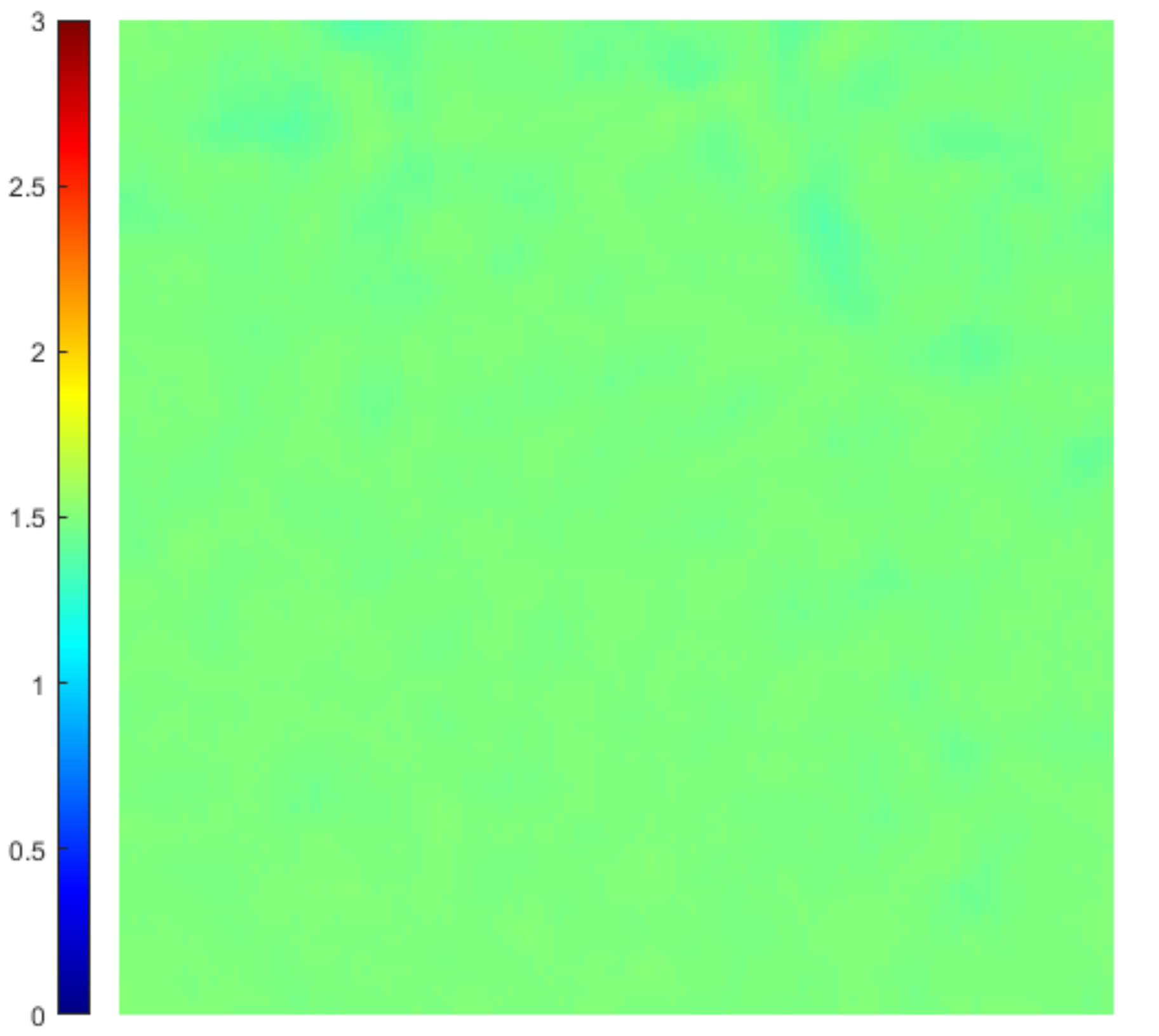}
	\caption{Average result of the impostor matching difference ${\Delta_{\text{impo}}}$ across ${10^4}$ independent experiments.}
	\label{fig:impo_mean}
\end{figure}

\begin{figure}[]
	\centering
	\includegraphics[width=0.94\linewidth]{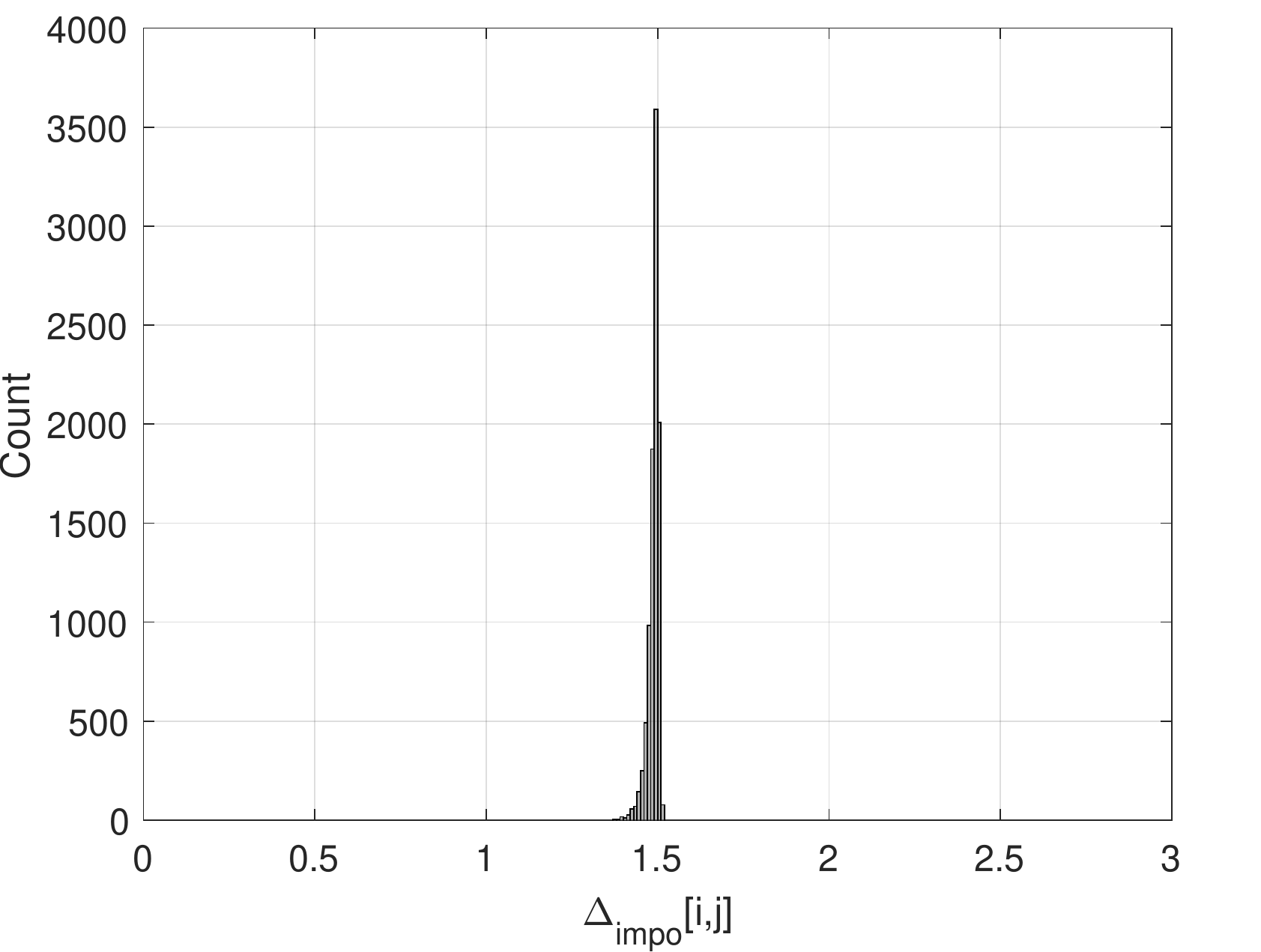}
	\caption{Histogram of the average ${\Delta_{\text{impo}}}$ shown in Figure~\ref{fig:impo_mean}. It is clear that the average of ${\Delta_{\text{impo}}[i,j]}$ concentrates around ${1.5}$.}
	\label{fig:histogram_impo_mean}
\end{figure}

\textbf{Non-constant expectation of ${\Delta_{\text{genu}}}$}. It's time to study the expectation of the genuine matching difference ${\Delta_{\text{genu}}}$. However, the TONGJI database \cite{zhang2017toward} we used in this paper only provides 20 palmprint images for each identity. Therefore, a maximum of ${C_{20}^{2} = 190}$ experiments can be conducted to estimate the expectation, which is far from sufficient. But we wish our experiments could at least make the trend clear. 

\begin{figure*}[t]
	\centering
	\includegraphics[width=0.96\linewidth]{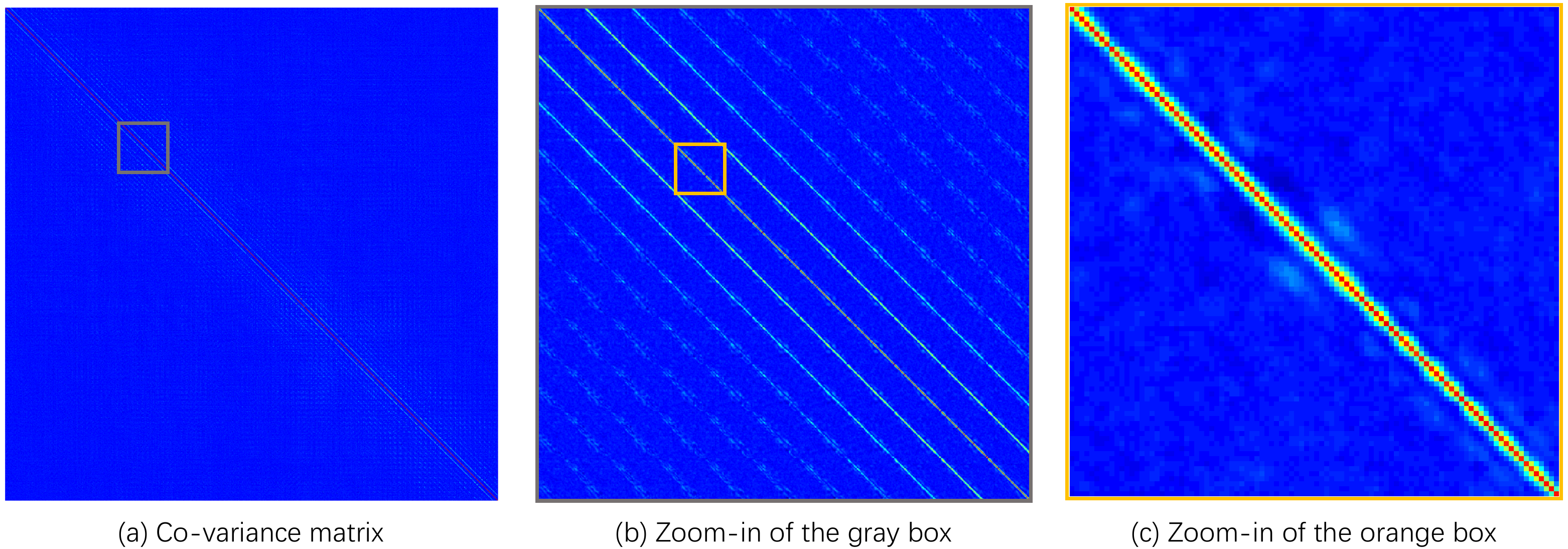}
	\caption{The empirical co-variance matrix of ${\text{Vec}(\Delta_{\text{impo}})}$ out of ${10^4}$ independent experiments. It can be seen that the co-variance matrix can be approximated by a block-circulant matrix, which mean that the variance of ${\Delta_{\text{impo}}}$ is stationary.}
	\label{fig:impo_var}
\end{figure*}

\begin{figure}[htb]
	\centering
	\includegraphics[width=1.0\linewidth]{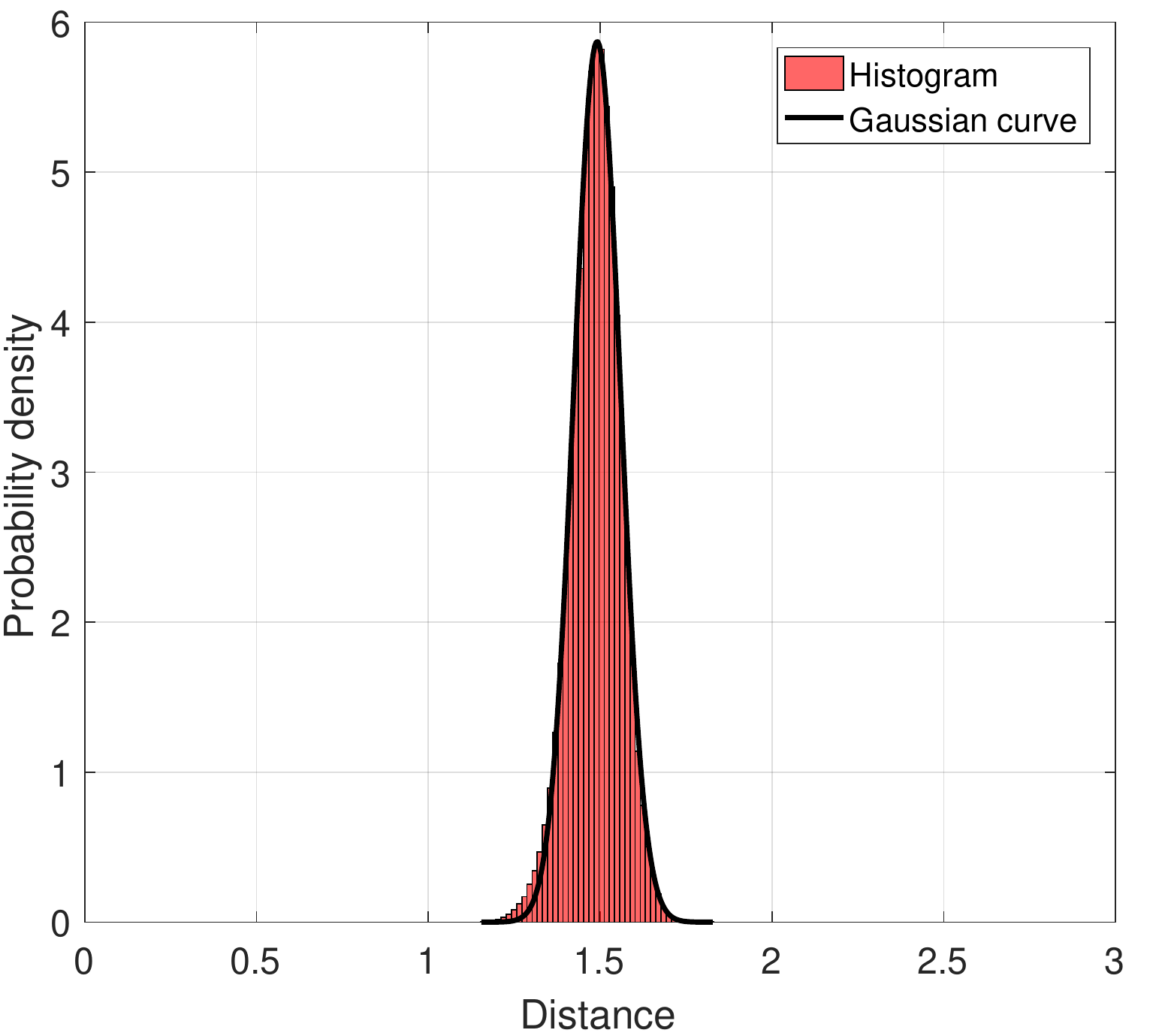}
	
	\caption{Histogram of the impostor matching distance out of ${4.9 \times 10^5}$ independent experiments. Black solid curve is the fitted Gaussian curve.}
	\label{fig:distribution_impo}
\end{figure}

Figure~\ref{fig:genu_mean} depicts the average result of ${\Delta_{\text{genu}}}$ out of 190 experiments. It can be seen that the expectation of ${\Delta_{\text{genu}}}$ is by no means constant. We find that ${\Delta_{\text{genu}}}$ can be roughly divided into two parts -- one part corresponds to the palm-line areas that contain significant orientation features, the other part corresponds to the flat areas between two palm lines. These non-palm-line areas contain no significant orientation, thus the features extracted from these areas are sensitive to noise and can be modeled as random. According to our experiments shown in Figure~\ref{fig:genu_mean}, the expectation of ${\Delta_{\text{genu}}}$ in palm-line areas is close to 0 (ideally, it should be 0), whereas the expectation of ${\Delta_{\text{genu}}}$ in non-palm-line areas is around 1.5, which equates the random results in ${\Delta_{\text{impo}}}$. This dichotomy structure may suggest that treating all the elements in ${\Delta_{\text{genu}}}$ in the same way, just as CompCode does, is not a good policy. 

\begin{figure}[htb]
	\centering
	\includegraphics[width=0.97\linewidth]{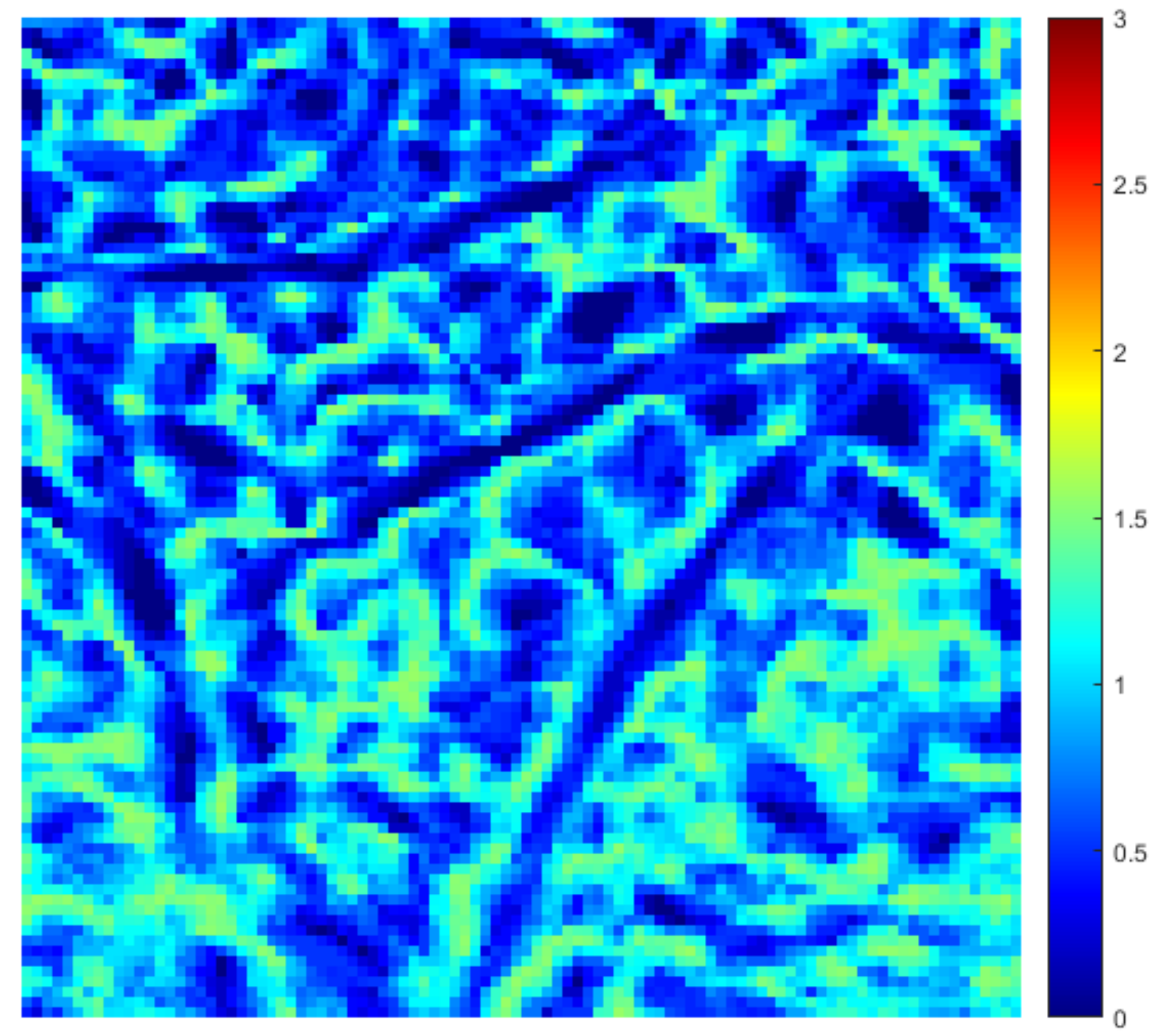}
	\caption{Average result of the genuine matching difference ${\Delta_{\text{genu}}}$ across 190 experiments. The average varies with position obviously.}
	\label{fig:genu_mean}
\end{figure}

\textbf{Non-stationary variance of ${\Delta_{\text{genu}}}$}. 
Finally, we will investigate the variance of ${\Delta_{\text{genu}}}$. Figure~\ref{fig:genu_var} displays the average result of ${\text{Var}[\text{Vec}(\Delta_{\text{genu}})]}$ out of 190 experiments. Although the estimation made by 190 experiments can hardly be precise, it could hopefully make the trend clear. Like the expectation of ${\Delta_{\text{genu}}}$, the co-variance of ${\text{Vec}(\Delta_{\text{genu}})}$ can also be roughly divided into two parts -- one part is the rows and columns that corresponds to the palm-line areas (the ``dark lines" in Figure~\ref{fig:genu_var}), the other part is the rows and columns that correspond to non-palm-line areas. We find that the rows and columns corresponding to palm-line areas have relatively low values that are close to 0 (ideally they should be 0), while the rows and columns corresponding to non-palm-line areas have relatively high values. This is to say the palm-line areas of ${\Delta_{\text{genu}}}$ have low variance, whereas the non-palm-line areas of ${\Delta_{\text{genu}}}$ have high variance. The non-stationary of ${\Delta_{\text{genu}}}$ also suggests the classifier with a weight vector ${[1,1,\cdots, 1]^{\intercal}_{N \times N}}$ is not optimal. 

Our investigation in the statistics of palmprints indicate that the classifier currently used in CompCode is not optimal. To show this, we need to test whether Eq.~\eqref{eq:10} holds when ${\mathbf{w} = [1,1,\cdots,1]^{\intercal}_{N \times N}}$, or whether the equivalence of Eq.~\eqref{eq:10},
\begin{equation} \label{eq:22}
(\Sigma_{\text{genu}} + \Sigma_{\text{impo}})\cdot \mathbf{w} \propto (\mathbf{\mu}_{\text{genu}} - \mathbf{\mu}_{\text{impo}}),
\end{equation}
holds when ${\mathbf{w} = [1,1,\cdots,1]^{\intercal}_{N \times N}}$ if the co-variance matrix ${\Sigma_{\text{genu}}}$ and ${\Sigma_{\text{impo}}}$ is singular. According to our investigations of statistics of palmprints, ${(\mathbf{\mu}_{\text{genu}} - \mathbf{\mu}_{\text{impo}}) \propto [1,1, 0, 1,0, \cdots, 0, 1]^{\intercal}_{N \times N}}$ where $1$ appears at the indexes corresponding to palm-line areas and $0$ appears at the indexes corresponding to non-palm-line areas. Let ${\mathbf{w} = [1,1,\cdots,1]^{\intercal}_{N \times N}}$, then we get ${\Sigma_{\text{impo}}\cdot \mathbf{w} \propto [1,1,\cdots, 1]^{\intercal}_{N \times N}}$ because ${\Sigma_{\text{impo}}}$ is stationary. To make Eq.~\eqref{eq:22} hold,  ${\Sigma_{\text{genu}} \cdot \mathbf{w}}$ must be proportional to a vector whose values at indexes that correspond to non-palm-line areas are ${-1}$. This is impossible, or at least can not be guaranteed, because most of the elements are positive in ${\Delta_{\text{genu}}}$. Therefore Eq.~\eqref{eq:22} can not hold when ${\mathbf{w} = [1,1,\cdots,1]^{\intercal}_{N \times N}}$. 

\begin{figure}
	\centering
	\includegraphics[width=0.98\linewidth]{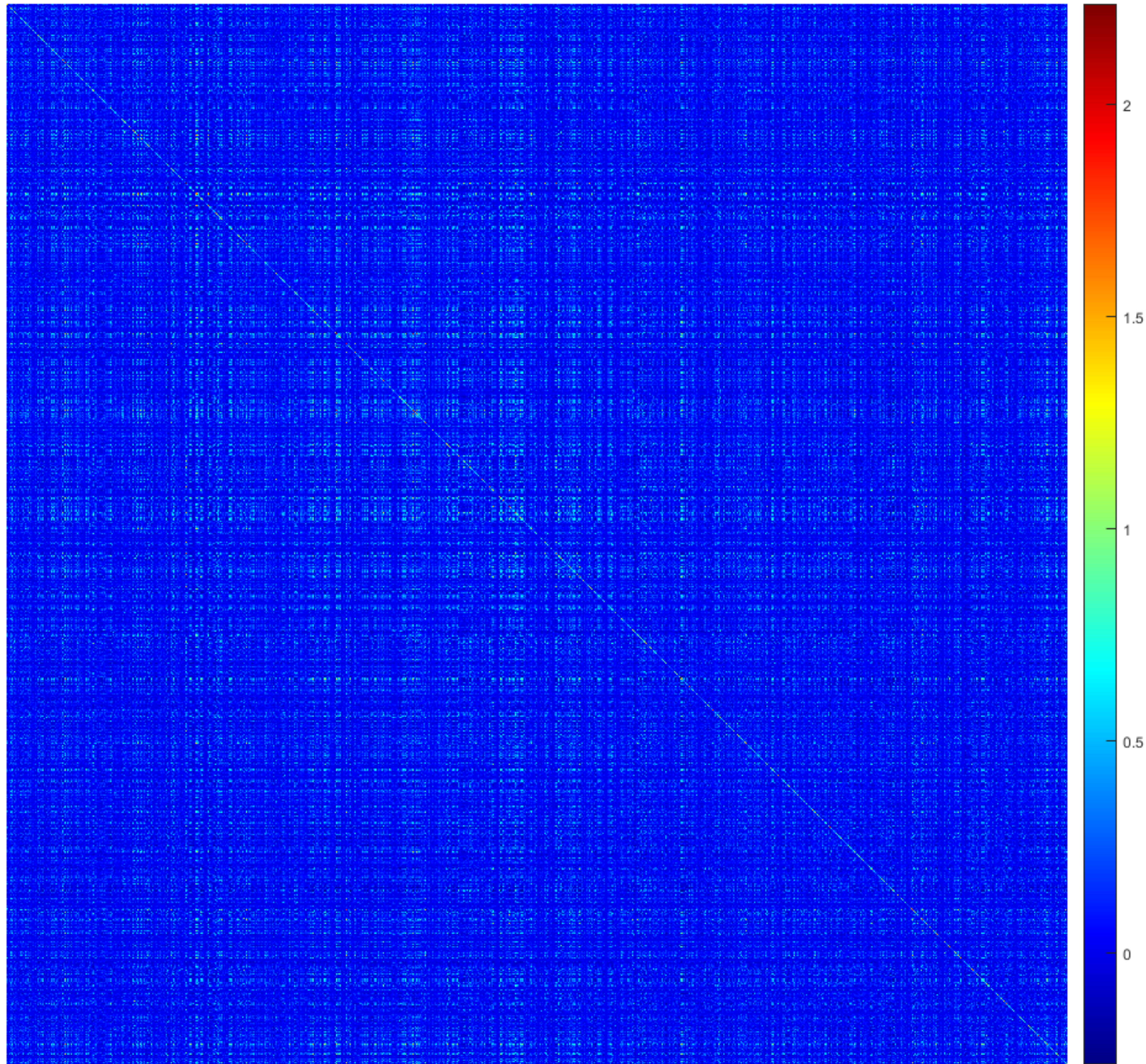}
	\caption{Average result of the co-variance matrix ${\text{Var}[\text{Vec}(\Delta_{\text{genu}})]}$ across 190 experiments. The lines or columns that correspond to palm-line areas have obviously low values.}
	\label{fig:genu_var}
\end{figure}

In the next section, we will propose a new palmprint recognition method called ``Class-Specific CompCode" (CSCC), which improves CompCode by excluding non-palm-line areas from matching. By this modification, the Fisher's criterion is satisfied by the proposed method.

%---------------------------------------------------------------------------------------------
%\section{Statistical Properties of Palmprints}

%---------------------------------------------------------------------------------------------
\section{Class-Specific CompCode (CSCC)} \label{sec:improvement}
In this section, we propose to improve CompCode in a way that the elements in the weight vector ${\mathbf{w}}$ corresponding to palm-line areas are set to $1$ whereas the elements corresponding to non-palm-line areas are set to $0$. This is to say we exclude non-palm-line areas from the matching. The improved CompCode is called Class-Specific CompCode (CSCC) because this method endows each class (identity) a specific classifier according to the profiles of the palm lines. 

Now, we will justify the modification by showing that this modification makes Eq.~\eqref{eq:22} holds. We have known that the right part of Eq.~\eqref{eq:22} is proportional to a vector ${[1,1,0,1,0, \cdots, 0,1]^{\intercal}_{N \times N}}$, where $1$ only appears at the indexes corresponding to palm-line areas. The rest thing is to prove that the left part of Eq.~\eqref{eq:22} is proportional to this vector. 

Let's first look at ${\Sigma_{\text{genu}}\cdot \mathbf{w}}$. As previously mentioned, the rows and columns of ${\Sigma_{\text{genu}}}$ corresponding to palm-line areas have low values that are close to 0 (ideally they should be 0), while the rows and columns corresponding to non-palm-line areas have relatively high values. Since we set ${\mathbf{w}[i]=1}$ if ${i}$ corresponds to palm-line areas and ${\mathbf{w}[i] = 0}$ if ${i}$ corresponds to non-palm-line areas, then ${\mathbf{w}}$ is orthogonal with each row of ${\Sigma_{\text{genu}}}$. Therefore ${\Sigma_{\text{genu}} \cdot \mathbf{w} = 0}$.

\begin{figure}[t]
	\centering
	\includegraphics[width=0.9\linewidth]{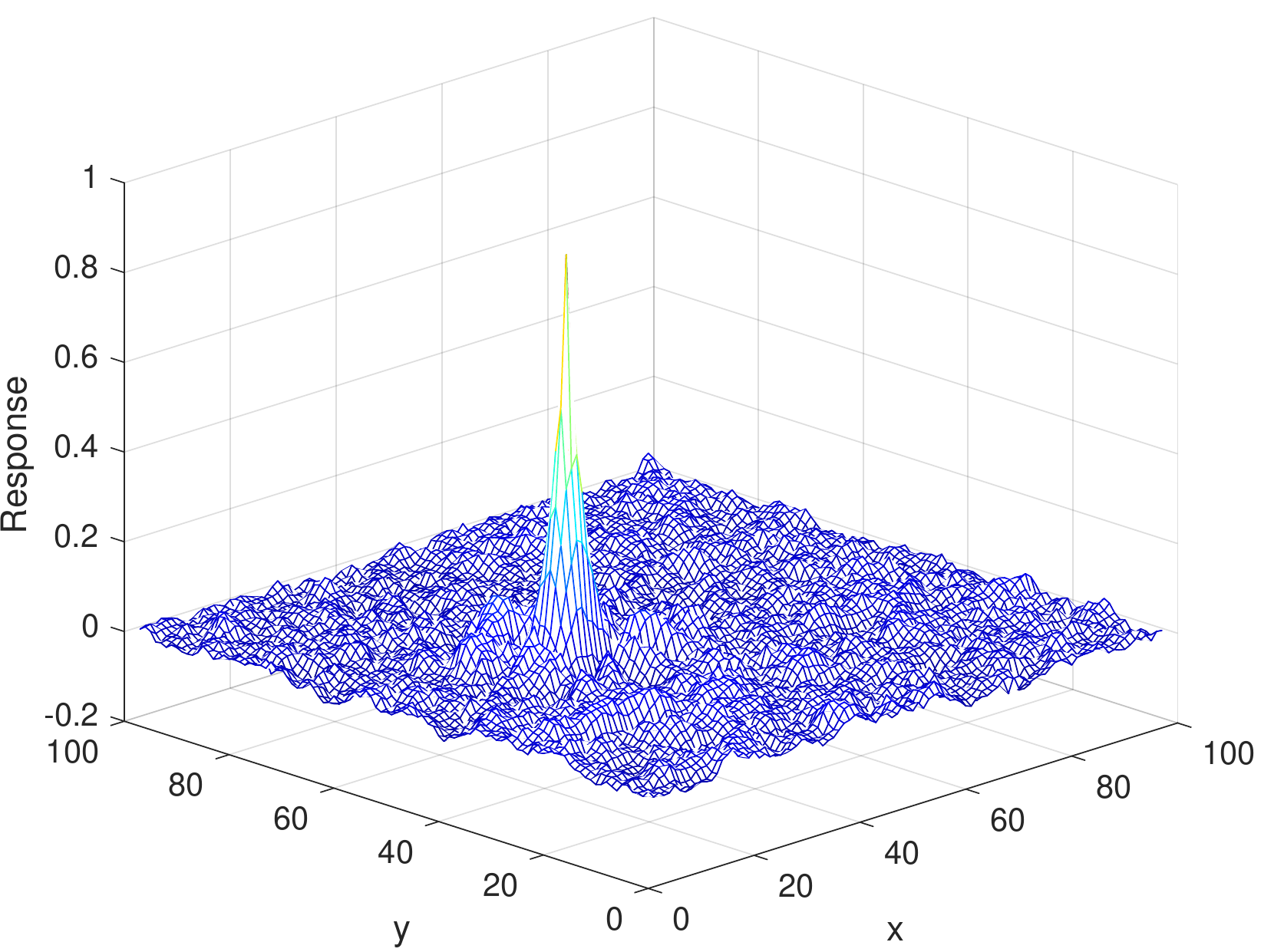}
	\caption{The block-circulant matrix ${\Delta_{\text{impo}}}$ is produced by this kernel.}
	\label{fig:kernel}
\end{figure}

Then, let's look at ${\Sigma_{\text{impo}} \cdot \mathbf{w}}$. From Eq.~\eqref{eq:15} and Eq.~\eqref{eq:16} we can see that ${\Sigma_{\text{impo}}}$ is produced by a kernel
\begin{equation} 
	K = \left [ 
			\begin{matrix}
					b^{0}_{0} & b^{0}_{1} & b^{0}_{2} & b^{0}_{3} & \cdots   & \cdots  & b^{0}_{N} ~\\
					b^{1}_{0} & b^{1}_{1} & b^{1}_{2} & b^{1}_{3} & \cdots   & \cdots  & b^{1}_{N} \\
					b^{3}_{0} & \cdots    & \cdots    & \cdots    & \cdots   & \cdots  & b^{3}_{N} \\
					          &           & \ddots    &           & \ddots   &         &           \\
							  &           & \ddots    &           & \ddots   &         &           \\
					b^{N}_{0} & b^{N}_{1} & \cdots    & \cdots    & \cdots   & b^{N}_{N-1}  & b^{N}_{N} 
			\end{matrix}
		\right] .
\end{equation}
By permuting the elements in a row of ${\Sigma_{\text{impo}}}$, we can get the 2-D kernel shown in Figure~\ref{fig:kernel}. It is shown that the kernel has a Laplacian profile. The support of the kernel is around ${5 \times 5}$. According to a little knowledge of linear algebra, we have
\begin{equation}
	\Sigma_{\text{impo}} \cdot \mathbf{w} = K \otimes W,
\end{equation}
where ${W}$ is the 2-D form of ${\mathbf{w}}$ (${\text{Vec}(W) = \mathbf{w}}$), and ${\otimes}$ represents 2-D convolution. This is to say ${\Sigma_{\text{impo}}}$ smooths ${\mathbf{w}}$ by a low pass filter. Since the support of ${K}$ is small, we have 
\begin{equation}
	\Sigma_{\text{impo}} \cdot \mathbf{w} = K \otimes W \approx W.
\end{equation}
By combining the above analysis, we obtain
\begin{equation}
	(\Sigma_{\text{genu}} + \Sigma_{\text{impo}}) \cdot \mathbf{w} \propto \mathbf{w}.
\end{equation}
According to the definition, ${\mathbf{w} \propto (\mathbf{\mu}_{\text{genu}} - \mathbf{\mu}_{\text{impo}})}$, thus
\begin{equation}
(\Sigma_{\text{genu}} + \Sigma_{\text{impo}}) \cdot \mathbf{w} \propto (\mathbf{\mu}_{\text{genu}} - \mathbf{\mu}_{\text{impo}}).
\end{equation}
This proves that the by modifying the weight vector, the proposed CSCC method meets Fisher's criterion. 

\begin{figure}
	\centering
	\subfigure[palmprint]{
		\includegraphics[width=0.14\textwidth]{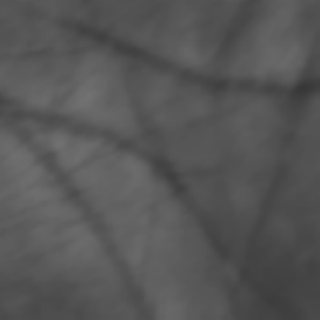}
	}
	\subfigure[Gabor response]{
		\includegraphics[width=0.14\textwidth]{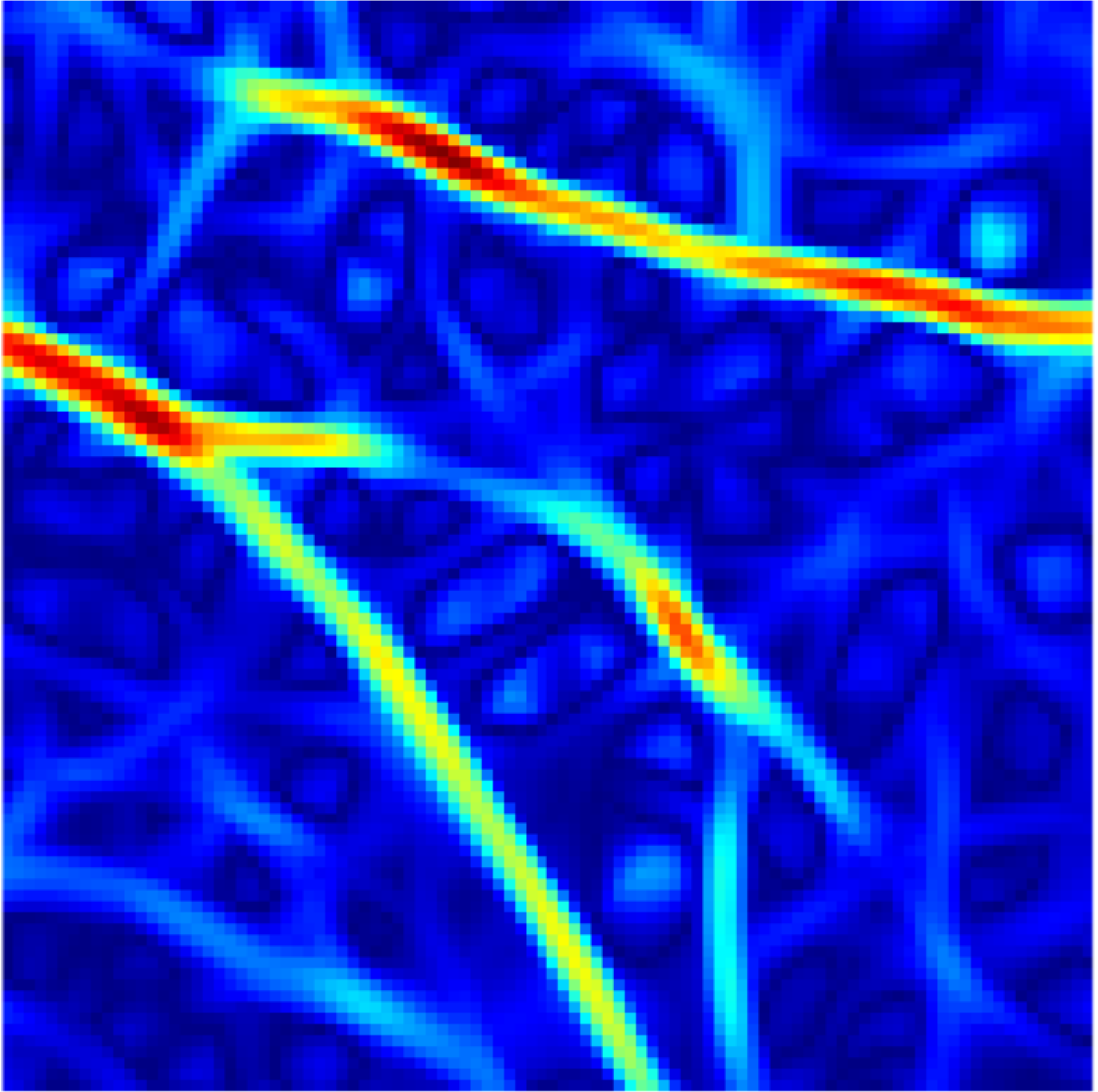}
	}
	\subfigure[palm lines]{
		\includegraphics[width=0.14\textwidth]{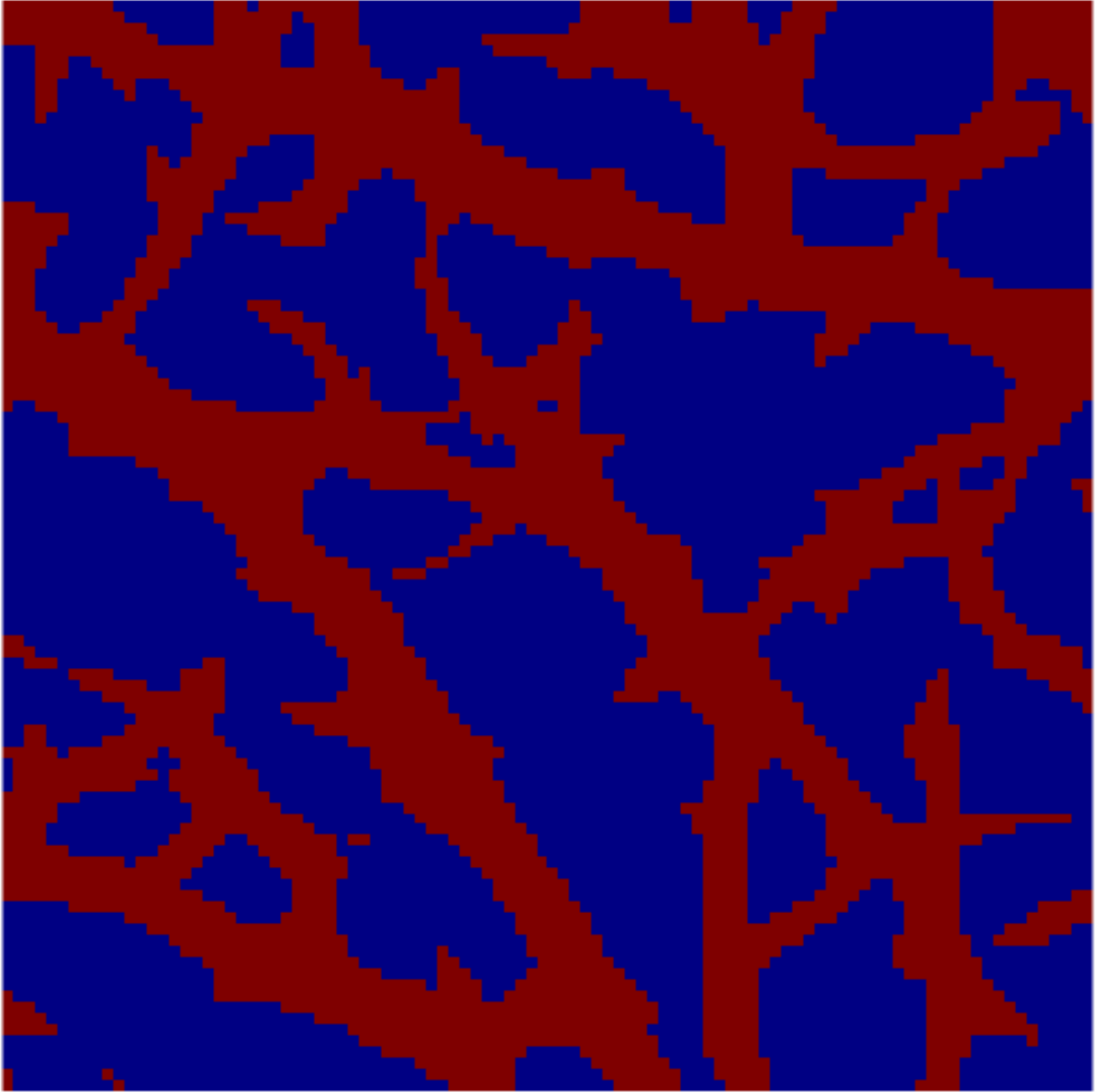}
	}
	\caption{Separate palm-line areas and non-palm-line areas by a threshold. (a) An example of a palmprint. (b) The filter response of six Gabor filters. (c) Binary mask of palm-line areas and non-palm-line areas.}
	
	\label{fig:palm_line}
\end{figure}
 
There remains a question of how to separate palm-line and non-palm-line areas. Based on our observation on Eq.~\eqref{eq:4}, the areas with big amplitude ${A}$ have strong filter responses, whereas flat areas will generate weak filter responses. So we can simply set a threshold to separate palm-line areas and non-palm-line areas. Figure~\ref{fig:palm_line} shows an example of separating palm-line areas and non-palm-line areas. The areas with strong filter responses are categorized as palm-line areas, whereas the areas with weak filter responses are categorized as non-palm-line areas. 

To further improve the performance, we propose to impose a non-linear mapping on the matching difference, i.e., 
\begin{equation}\label{eq:27}
	f(\Delta[i,j]) = \exp (-k\cdot \Delta[i,j]), \quad k > 0.
\end{equation}
Our motivation is to find a more discriminative coding scheme. For two-choice decision tasks, the decidability index ${d^{'}}$ measures how well separated the two distribution (e.g., genuine matching distance ${d_{\text{genu}}}$ and impostor matching distance ${d_{\text{impo}}}$). If their two means are ${\mu_1}$ and ${\mu_2}$, and their two standard deviations are ${\sigma_1}$ and ${\sigma_2}$, then ${d^{'}}$ is defined as \cite{daugman2003random}
\begin{equation}
d^{'} = \frac{|\mu_1 - \mu_2|}{\sqrt{(\sigma_1^{2}+\sigma_2^{2})/2}}.
\end{equation}
It is difficult to determine an optimal coding scheme because it requires the full knowledge about the distribution of ${\Delta_{\text{genu}}}$ and ${\Delta_{\text{impo}}}$. But it is easy to prove that under the ideal case, i.e., ${\Delta_{\text{genu}} = 0}$, the currently used coding scheme is not optimal at all. Actually, under the ideal case, the optimal coding scheme is 0-1 coding. The proof of this conclusion is presented in appendix~\ref{sec:appendix}. 

However, in practice, the genuine matching is not perfect -- a great amount of pixels are not exactly matched in the genuine matching due to various factors like misalignment, deformation, noise, etc. Directly applying 0-1 coding in practical palmprint recognition would significantly degrade the performance. There is a trade-off between decidability and robustness. In this paper, we use a negative exponent function (Eq.~\eqref{eq:27}) to make this trade-off because the negative exponent function approaches 0-1 as ${k}$ increases and tends to a linear function as ${k}$ decreases. 

In short, the proposed method CSCC improves CompCode in two folds. First, CSCC excludes non-palm-line areas from matching to meet Fisher's criterion. Second, by imposing a non-linear mapping on the competitive code, CSCC represents palmprints in a more discriminative way. In the next section, we will verify the effectiveness of the proposed method. Besides, we will also demonstrate that these two improvement strategies also benefit other coding-based methods.

%---------------------------------------------------------------------------------------------
\section{Experiments} \label{sec:experiments}
In this section, we will first verify the effectiveness of the proposed method CSCC. Then, we will extend our improvement strategies to other coding based methods. Experiments show that our improvement strategies also benefit these methods.

\subsection{Experimental Setup}
Our experiments are conducted on two public databases, i.e., TONGJI \cite{zhang2017toward} and IITD \cite{kumar2008incor}. TONGJI database collects palmprint images from 300 volunteers, including 192 males and 108 females. The palmprints are collected in two separate sessions. In each session, the subject was asked to provide 10 images for each palm. Therefore, the database totally contains 12000 images captured from 600 different palms. IITD database \cite{kumar2008incor} consists of 2300 palmprint images from 230 subjects. Each subject was asked to provide 5 images of each palm. The palmprint images are captured in a contactless way, which results in severe misalignment and deformation. 

In our experiments on TONGJI database, we use the whole database for genuine matching, which generates ${1.14 \times 10^5}$ comparisons, and 4000 images for impostor matching, which generates ${7.96 \times 10^6}$ comparisons. In our experiments on IITD database, we manually drop severely misaligned palmprint images and use 1875 images in the database, which generates ${3750}$ genuine comparisons and about ${1.753 \times 10^6}$ impostor comparisons. 
We empirically exclude ${30\%}$ areas of palmprints from matching and set the non-linear mapping factor ${k}$ as ${1.0}$. 
The receiver operating characteristic (ROC) curve, equal error rate (EER), and false reject rate (FRR) at a specific false accept rate (FAR) are used to evaluate the performance of the palmprint recognition methods. 

%CompCode, as well as other coding based methods, are sensitive to the misalignment between two palmprints due to its pixel-to-pixel comparison. Because of imperfect pre-processing, we need to vertically and horizontally translate one of the palmprint features and then perform the matching again. The minimum value obtained from multiple matching is considered to be the final matching distance. 
%In our experiments, the translation range is designed as from ${-16}$ to ${16}$ on TONGJI database and ${-24}$ to ${24}$ on IITD database. The translation, on the one hand, decreases genuine matching distance, but on the other hand, it makes the distribution of the impostor matching distance not Gaussian (see the distribution of the impostor matching distance in Figure~\ref{fig:compcode_dis}). Actually, the final impostor matching distance we obtained from translation is the minimum order statistic of a set of i.i.d. random variables. Detailed analysis of the minimum order statistics can be found in any textbook of statistics, but roughly speaking the distribution of the minimum order statistics leans to the left with a ``left tail".

\subsection{Results}
Figure~\ref{fig:roc-compcode} shows the ROC curves of CompCode and the proposed Class-Specific CompCode (CSCC). One can see that for any given FAR, the genuine accept rate (GAR) of CSCC is always higher than that of CompCode. Table~\ref{table:2} lists the values of FRR and the values of EER of CompCode and the proposed CSCC. We use FRR${_{-6}}$ (FRR${_{-5}}$) to denote the false reject rate when FAR is ${10^{-6}}$ (${10^{-5}}$). Compared with CompCode, the proposed CSCC achieves lower EER value and lower FRR value. These results demonstrate the effectiveness of the proposed method CSCC. 

It is also important to analyze the computational complexity of the proposed CSCC. The proposed CSCC improves CompCode in two folds. The first is excluding non-palm-line areas from matching by setting a threshold, the second is imposing a non-linear mapping on the matching difference. Threshold operation is computationally efficient, the time consuming non-linear mapping can be efficiently implemented by table looking because the domain of this operation is just four discrete points (i.e., 0,1,2,3). Thus our improvements will not induce significant computational overhead. Table~\ref{table:3} compares the run time of CompCode and the proposed CSCC (both are Matlab code). It can be seen that our method is just 2{\,}ms slower than CompCode. 

\subsection{Ablation Study}
To see the contribution of each improvement strategy, we conduct an ablation study on TONGJI database. The results are shown in Table~\ref{table:4}. We use M-CC to denote the improved CompCode by excluding non-palm-line areas from matching and use E-CC to denote the improved CompCode by implementing a non-linear mapping. It can be seen that both of the two improvement strategies boost the performance. By combining these two strategies, the proposed CSCC achieves the best result.

\begin{figure}[t]
	\centering
	\subfigure{
		\includegraphics[width=0.96\linewidth]{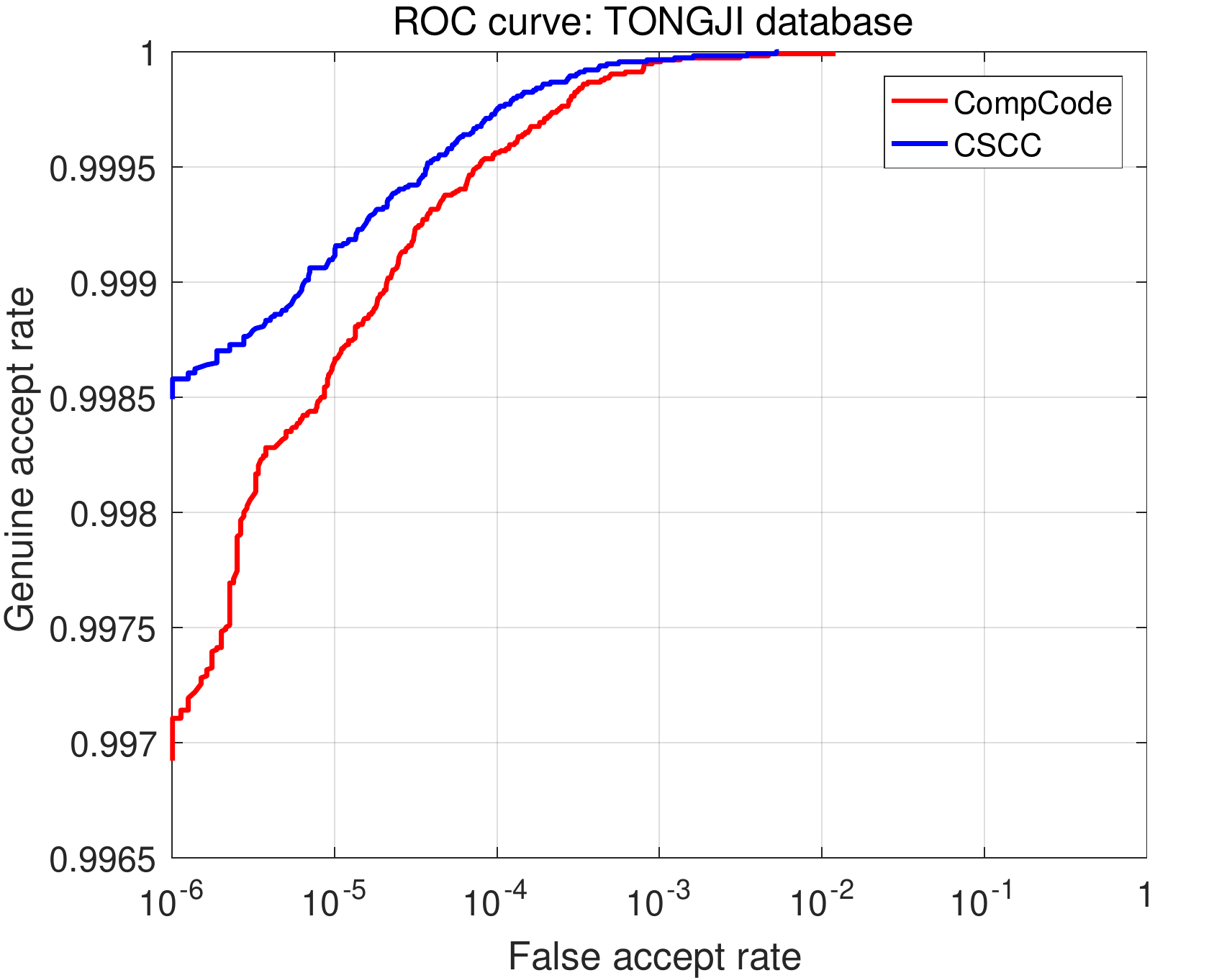}
	}\\
	\subfigure{
		\includegraphics[width=0.96\linewidth]{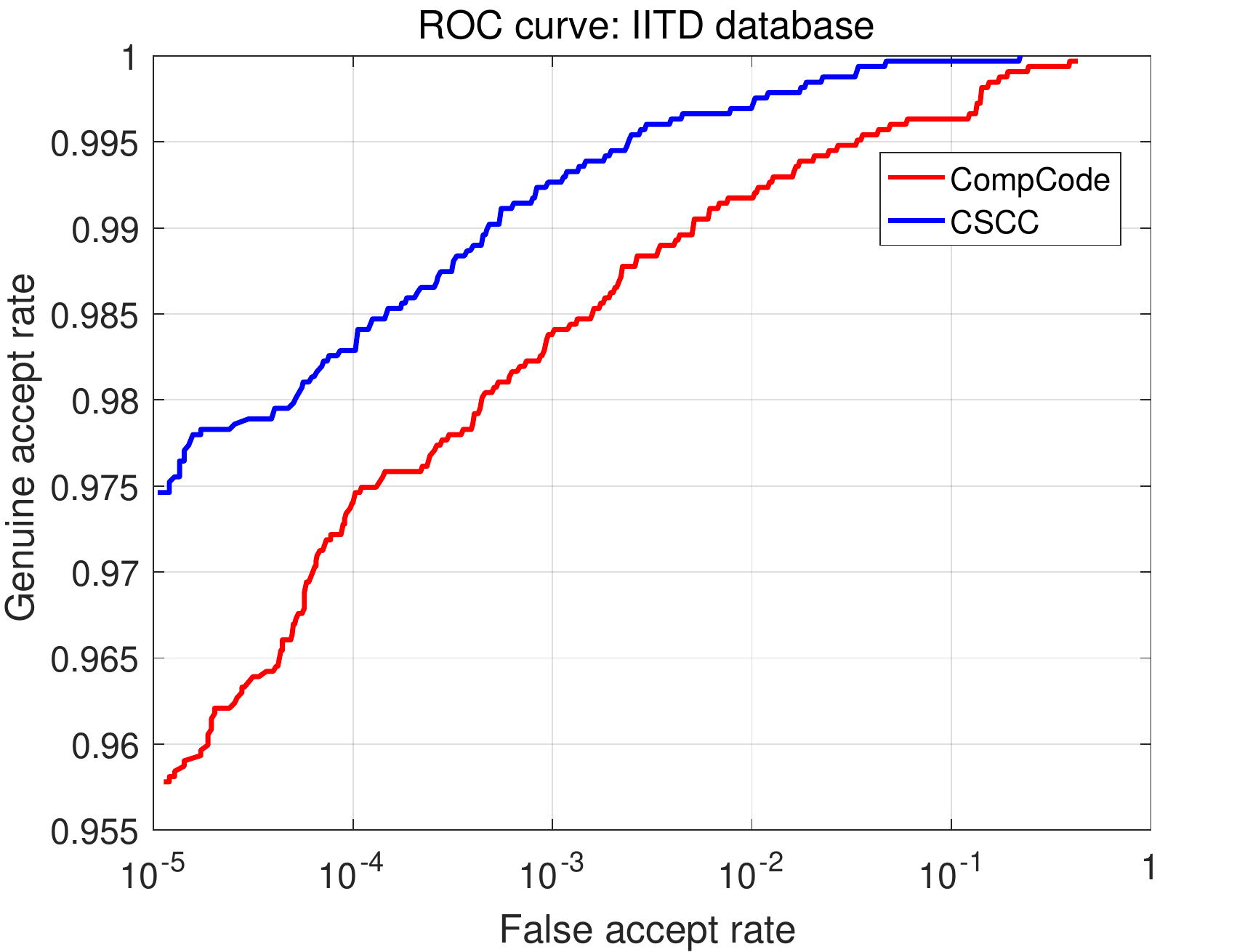}
	}
	
	\caption{ROC curves of CompCode and the proposed CSCC on TONGJI database and IITD database, respectively.}
	\label{fig:roc-compcode}
\end{figure}

\begin{table}[htbp]
	\centering
	\small
	\renewcommand{\arraystretch}{1.3}
	\caption{Comparison of error rates of CompCode and the proposed CSCC on TONGJI database and IITD database. We use FRR${_{-6}}$ (FRR${_{-5}}$) to denote the false reject rate when FAR is ${10^{-6}}$ (${10^{-5}}$).}
	\vspace{2pt}
	\begin{tabular}{| p{1.7cm}<\centering | p{1.21cm}<\centering | p{1.21cm}<\centering | p{1.21cm}<\centering | p{1.21cm}<\centering |}
		
		\hline
		\multirow{2}{*}{Method}	& \multicolumn{2}{c|}{TONGJI} & \multicolumn{2}{c|}{IITD}		\\
		\cline{2-5}
			                    & FRR${_{-6}}$  &  EER        &  FRR${_{-5}}$   &  EER          \\
		\hline				 
		CompCode                & 0.308\%       &  0.024\%    &  4.25\%         &  0.84\%       \\
		\hline
		\textbf{CSCC}           & \textbf{0.151\%} & \textbf{0.017\%} & \textbf{3.50\%} & \textbf{0.51\%}  \\
		\hline  			 
		
	\end{tabular}
	
	\label{table:2}
\end{table}

\begin{table}[htbp]
	\centering
	\small
	\renewcommand{\arraystretch}{1.3}
	\caption{Comparison of run time of CompCode and the proposed CSCC.}
	\vspace{2pt}
	\begin{tabular}{| p{2cm}<\centering | p{2cm}<\centering | p{2cm}<\centering |}
		
		\hline
		Method         &         CompCode       &         CSCC      		\\
		\hline
		Run time       &         18{\,}ms       &         20{\,}ms          \\
		\hline				  			 
		
	\end{tabular}
	
	\label{table:3}
\end{table}

\begin{table}[htbp]
	\small
	\centering
	\renewcommand{\arraystretch}{1.3}
	\caption{Ablation study of two improvement strategies. M-CC denotes excluding non-palm-line areas from matching, and E-CC denotes imposing a non-linear mapping on the matching difference.}
	\vspace{2pt}
	\begin{tabular}{| L{0.8cm} | p{1.5cm}<{\centering} | p{1.4cm}<{\centering} | p{1.4cm}<{\centering} | p{1.4cm}<{\centering}|}
		\hline
			  		   &  CompCode   &  M-CC       &  E-CC       & \textbf{CSCC} \\ 
		\hline
		 FRR${_{-6}}$  &  0.308\%    &  0.300\%    &  0.187\%    & \textbf{0.151\%} \\
		\hline
		 FRR${_{-5}}$  &  0.140\%    & 0.119\%     & 0.097\%     & \textbf{0.086\%} \\
		\hline
		 EER           &  0.024\%    &  0.018\%    &  0.018\%    & \textbf{0.017\%} \\
		\hline
	\end{tabular}

	\label{table:4}
\end{table}

\subsection{Experiments on other coding-based methods}
Previous analysis has uncovered two facts that will also benefit other coding-based palmprint recognition methods. One is that palm-line areas and non-palm-line areas should not be treated indiscriminately, the other is that the coding scheme should make a trade-off between decidability and robustness. In this subsection, we will apply our improvement strategies on several other coding-based methods, e.g., RLOC \cite{jia2008palmprint}, DOC \cite{fei2016double-orientation}, HOC \cite{fei2016half}, and DRCC \cite{xu2018discriminative}. 

Our experiments are conducted on TONGJI database \cite{zhang2017toward}. Figure~\ref{fig:roc} displays the ROC curves of four coding-based methods and their corresponding improved versions. We use the prefix ``CS-" to distinct the improved version from the basic version. From the results, it can be seen that almost for any given FAR, the genuine accept rate (GAR) of the improved version is always higher than that of the basic version. This shows that our improvement strategies benefit these coding-based methods. 

\begin{figure*}[htbp]
	\centering
	
	\subfigure{
		\includegraphics[width=0.48\linewidth]{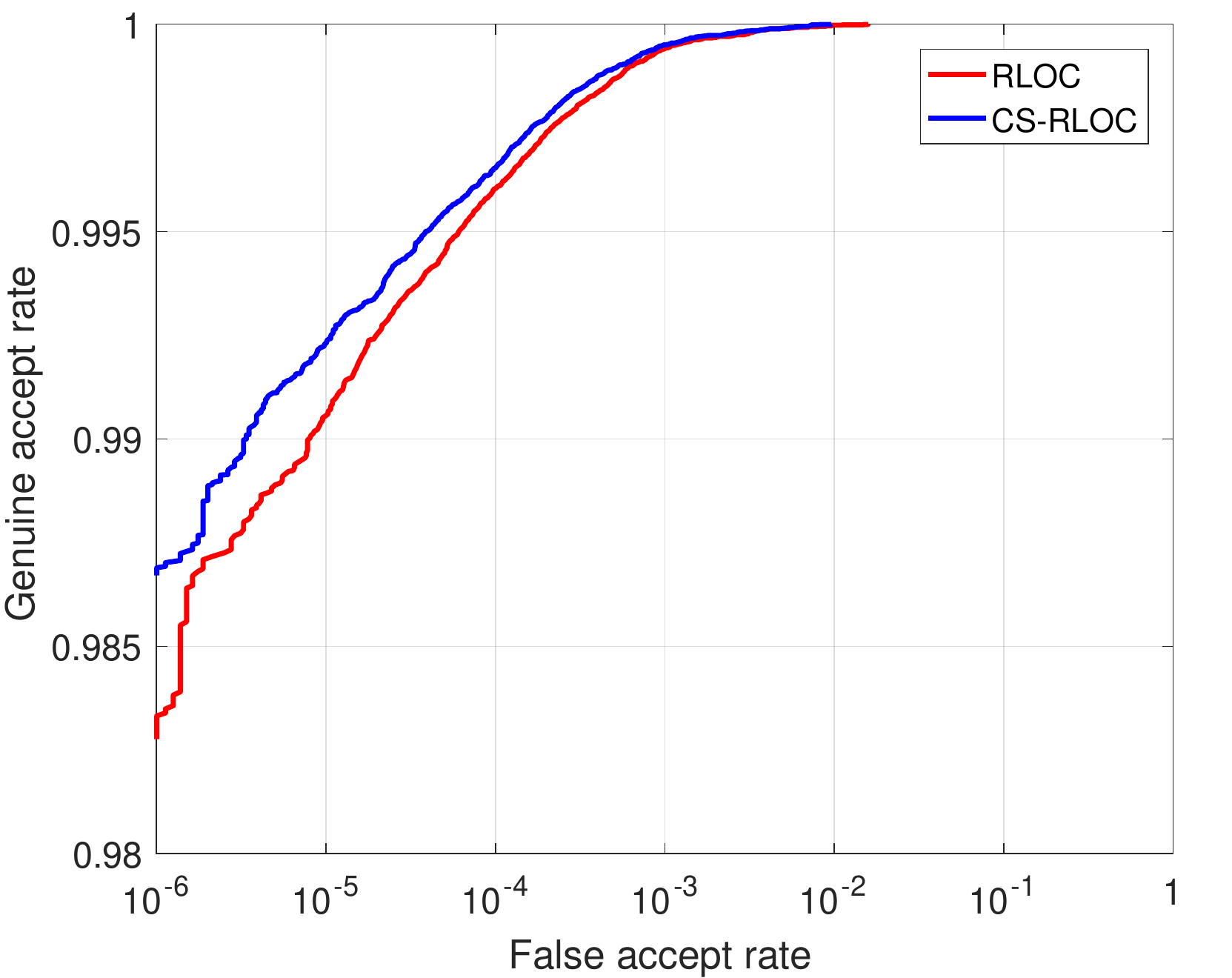}
	}
	\subfigure{
		\includegraphics[width=0.48\linewidth]{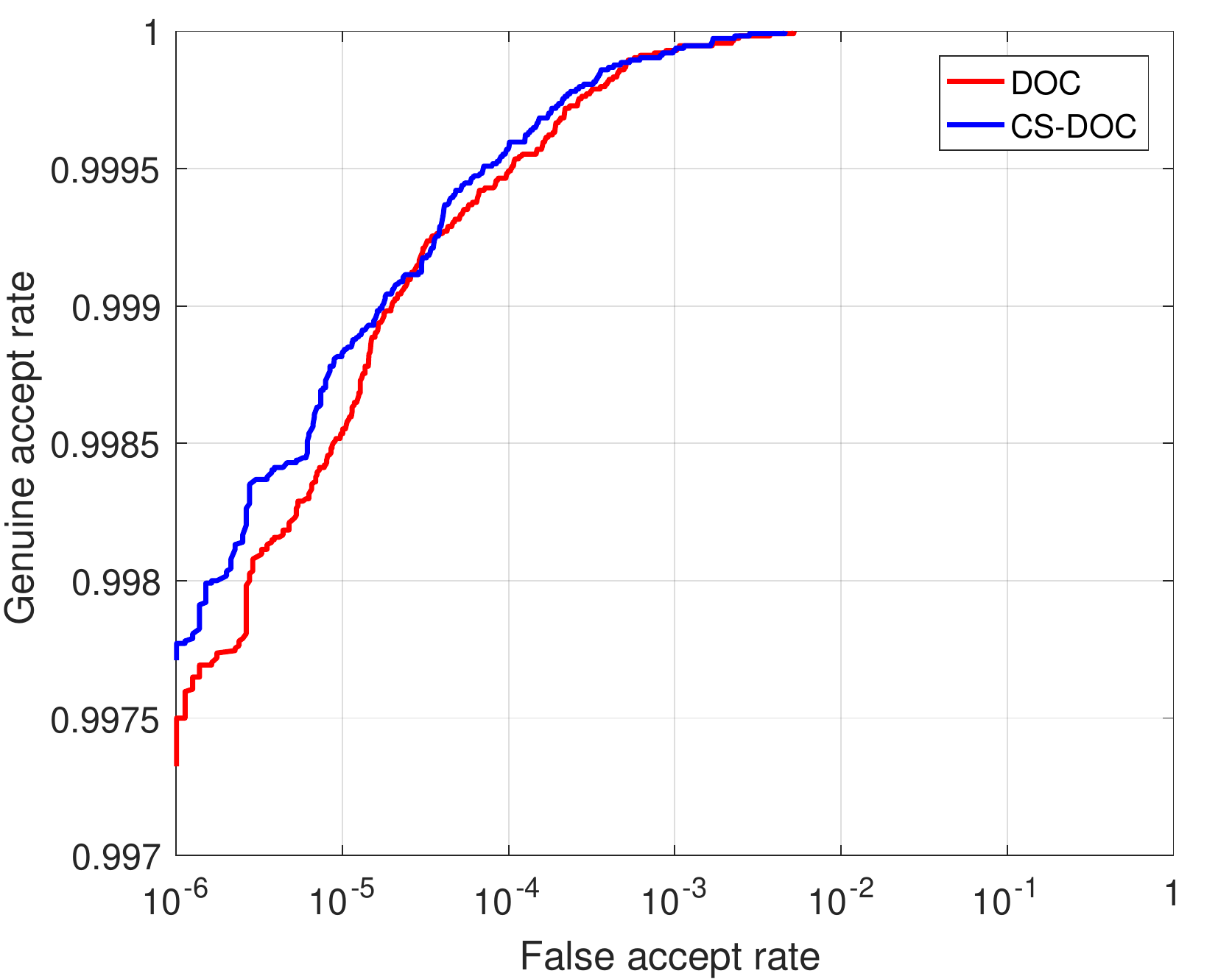}
	}

	\vspace{12pt}
	
	\subfigure{
		\includegraphics[width=0.48\linewidth]{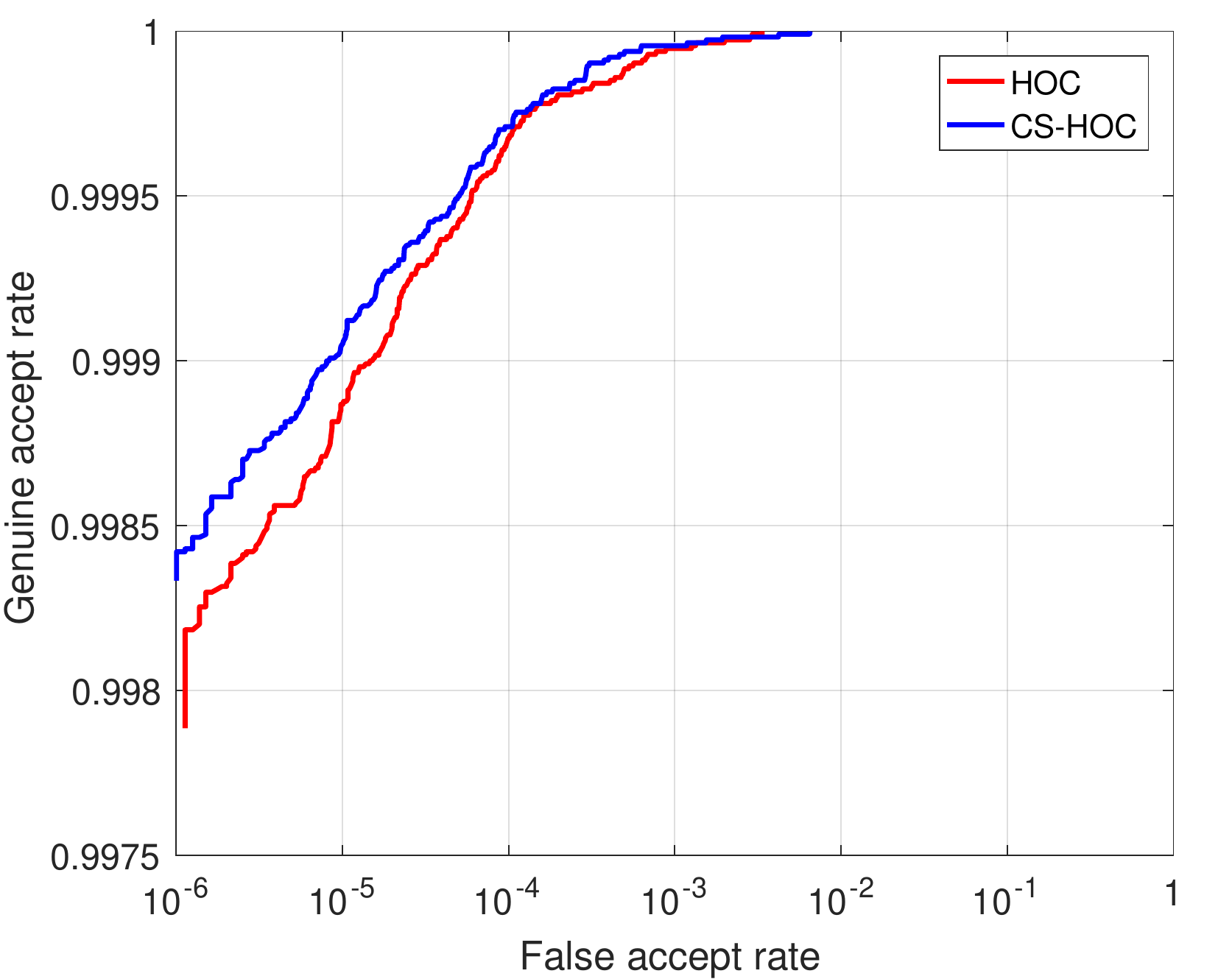}
	}
	\subfigure{
		\includegraphics[width=0.48\linewidth]{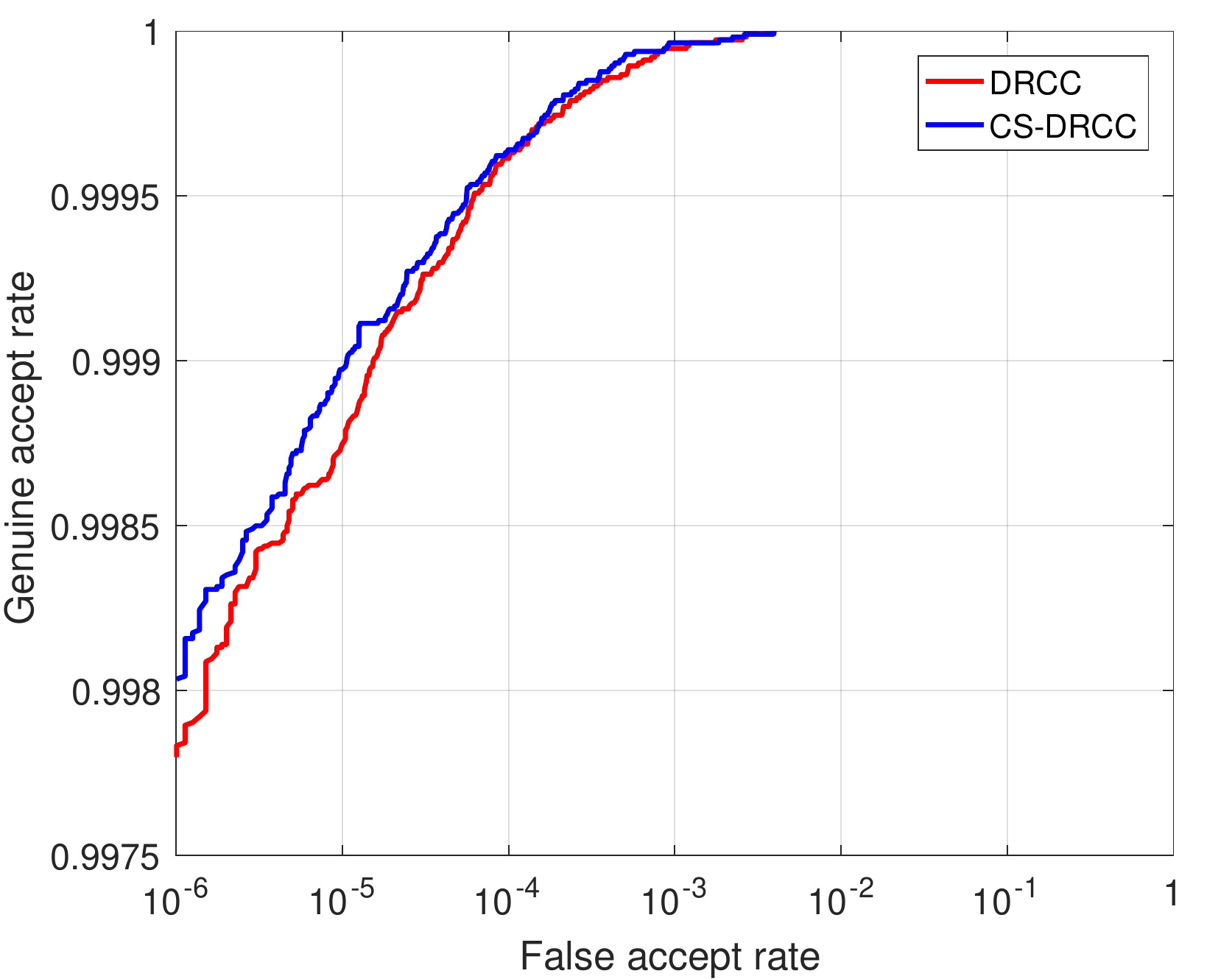}
	}

	\vspace{4pt}
	\caption{ROC curves of four coding-based methods and their corresponding improved versions. For example, RLOC is a coding-based method and CS-RLOC is the improved version by our improvement strategies. One can see that almost for any given FAR, the genuine accept rate of the improved versions are always higher than those of the basic versions.}
	
	\label{fig:roc}
\end{figure*}

\section{Conclusion}\label{sec:conclusion}
In this paper, we provide a detailed analysis of CompCode from the perspective of Linear Discriminant Analysis (LDA). Our analysis shows that the classifier used in CompCode is not optimal in the sense of Fisher's criterion. To improve CompCode, we propose to excluding non-palm-line areas from matching. The justification of this modification is that in this way the improved CompCode (CSCC) satisfies Fisher's criterion. In addition, a non-linear mapping is used to further enhance accuracy. Our experiments demonstrate the effectiveness of the proposed method CSCC and the benefits of these improvement strategies for other coding-based methods. 

We find that coding-based methods are very sensitive to misalignment between two palmprints due to its pixel-to-pixel matching. Currently, all these methods compensate for the misalignment by translating one of the palmprints horizontally and vertically at the matching stage, which can speak of a lack of alternatives. 
Therefore, in the future, we plan to investigate the registration between two palmprints. Besides, the methods that deal with palmprint deformation are also of great importance.



\bibliographystyle{elsarticle-num} 
\bibliography{reference}


\appendix

\section{Optimal Coding Scheme} \label{sec:appendix}
Section~\ref{sec:improvement} states that ideally, the optimal coding scheme is 0-1 coding. Here we will give a proof of it. 

\newtheorem{lemma}{Lemma}
\begin{lemma}
\rm{Let ${z}$ be a function defined on probability space ${\{0,1,2,3\}}$. Let}
\begin{equation} \label{eq:a1}
	d^{'} = \frac{1}{\sqrt{2}}\frac{|z(0) - \mu_2|}{\sigma_2},
\end{equation}
where
\begin{equation} \label{eq:a2}
	\begin{aligned}
		\mu_2    &= \text{E}(z), \\
		\sigma_2 &= \sqrt{\text{Var}(z)},
	\end{aligned}
\end{equation}
then, ${d^{'}}$ reaches the maximum if and only if ${z(0) \neq z(1) = z(2) = z(3)}$.
\end{lemma}

\begin{proof}
Let's first define the distribution function of the probability space ${\{0,1,2,3\}}$ as ${P}$. Then, the probabilities that 0,1,2, 3 appear can be denoted as ${P_0}$, ${P_1}$, ${P_2}$, and ${P_3}$, respectively. We have
\begin{equation}\label{eq:a3}
	\begin{aligned}	
		\mu_2      = & P_{0}z(0)+P_{1}z(1)+P_{2}z(2)+P_{3}z(3), \\
		\sigma_2^2 = & P_{0}(z(0)-\mu_2)^{2}+P_{1}(z(1)-\mu_2)^{2}+P_{2}(z(2)-\mu_2)^{2} \\
				     & \quad + P_{3}(z(3)-\mu_2)^{2}. 
	\end{aligned}
\end{equation}
Without loss of generality, we set ${z(0)=1}$ and prove that ${d^{'}}$ reaches maximum if and only if ${z(1) = z(2) = z(3) \neq 1}$. For convenience, we will use the square of ${d^{'}}$ and omit the constant factor. Combining Eq.~\eqref{eq:a2} and Eq.~\eqref{eq:a3}, we can get 
\begin{equation}\label{eq:a4}
	\begin{aligned}
		  & \frac{(z(0) - \mu_2)^2}{\sigma_2^2} \\
		= & (1 - \mu_2)^{2} / (P_{0}(1-\mu_2)^{2} + \\
	\quad & P_{1}(z(1)-\mu_2)^{2}+P_{2}(z(2)-\mu_2)^{2} + P_{3}(z(3)-\mu_2)^{2}) \\
		= & 1 / (P_{0}+ \\
	\quad & \frac{P_{1}(z(1)-\mu_2)^{2}}{(1-\mu_2)^{2}} + 
				\frac{P_{2}(z(2)-\mu_2)^{2}}{(1-\mu_2)^{2}} + 
				\frac{P_{3}(z(3)-\mu_2)^{2}}{(1-\mu_2)^{2}}).		      
	\end{aligned}
\end{equation}
According to Jensen's inequality, 
\begin{equation}\label{eq:a5}
	\begin{aligned}
		&\frac{P_{1}(z(1)-\mu_2)^{2}}{(1-\mu_2)^{2}}+
		 \frac{P_{2}(z(2)-\mu_2)^{2}}{(1-\mu_2)^{2}}+
		 \frac{P_{3}(z(3)-\mu_2)^{2}}{(1-\mu_2)^{2}} \\
   \geq & \frac{P_{0}^{2}}{1-P_{0}}.
	\end{aligned}
\end{equation}
Equality holds if and only if ${z(1) = z(2) = z(3) \neq 1}$. Thus, 
\begin{equation}\label{eq:a6}
	d^{'}=\frac{1}{\sqrt{2}}\frac{|z(0)-\mu|}{\sigma} \leq 
     		 \frac{1}{\sqrt{2}}\sqrt{\frac{1-P_{0}}{P_{0}}}.
\end{equation}
Equality holds if and only if ${z(0) \neq z(1) = z(2) = z(3)}$.
\end{proof}

\end{document}